%% file: ms.tex
\documentclass{article}

\usepackage[preprint,nonatbib]{neurips_2019}

\usepackage[abspath]{currfile}

\usepackage{geometry}

\usepackage{times}

\input{arxiv-preamble}

\def\titlename{Selecting the independent coordinates of manifolds with large aspect ratios}

\title{\titlename}
\author{
	Yu-Chia Chen \\
	Department of Electrical \& Computer Engineering \\
	University of Washington \\
	Seattle, WA 98195 \\
	\texttt{yuchaz@uw.edu} \\
	\And 
	Marina Meil\u{a} \\
	Department of Statistics \\
	University of Washington \\
	Seattle, WA 98195 \\
	\texttt{mmp2@uw.edu} \\
}

\begin{document}
\maketitle

\begin{abstract}
Many manifold embedding algorithms fail apparently when the data manifold has a large aspect ratio (such as a long, thin strip). Here, we formulate success and failure in terms of finding a smooth embedding, showing also that the problem is pervasive and more complex than previously recognized. Mathematically, success is possible under very broad conditions, provided that embedding is done by carefully selected eigenfunctions of the Laplace-Beltrami operator $\Delta$. Hence, we propose a bicriterial {\em Independent Eigencoordinate Selection (IES)} algorithm that selects smooth embeddings with few eigenvectors. The algorithm is grounded in theory, has low computational overhead, and is successful on synthetic and large real data.
\end{abstract}

\input{arxiv-motivation}

\input{arxiv-introduction}

\input{arxiv-algorithm}

\input{arxiv-embedding-dim-check}

\input{arxiv-theory}
\input{arxiv-experiments}

\input{arxiv-conclusion}

\section*{Acknowledgements}
The authors acknowledge partial support from the U.S. Department of Energy's Office of Energy Efficiency and Renewable Energy (EERE) under the Solar Energy Technologies Office Award Number DE-EE0008563 and from the NSF DMS PD 08-1269 and NSF IIS-0313339  awards. They are grateful to the Tkatchenko and Pfaendtner labs and in particular to Stefan Chmiela and Chris Fu for providing the molecular dynamics data and for many hours of brainstorming and advice. 

\section*{Disclaimer}
The views expressed herein do not necessarily represent the views of the U.S. Department of Energy or the United States Government.

\bibliographystyle{alpha}
\input{arxiv-eigencoords.bbl}

\input{arxiv-supplementary-mat}

\end{document}

%% file: arxiv-preamble.tex
\usepackage{graphicx}
\usepackage{amssymb,amsmath}
\usepackage{latexsym}
\usepackage{pst-node,pst-tree}
\usepackage{xcolor}
\usepackage{algorithmic}
\usepackage[linesnumbered,ruled]{algorithm2e}
\usepackage{bm}
\usepackage{setspace}
\usepackage{subfig}
\usepackage{booktabs}
\usepackage{amsthm}
\usepackage{url}
\usepackage[version=4]{mhchem}
\usepackage{wrapfig}
\usepackage{textcomp}
\usepackage{changepage}

\definecolor{darkred}{rgb}{0.5,0,0.1}
\definecolor{windowblue}{rgb}{0.21,0.47,0.75}

\newtheorem{lemma}{Lemma}

\newtheorem{theorem}[lemma]{Theorem}

\newtheorem{assu}{Assumption}

\renewenvironment{proof}{{\bfseries Proof.}}{\hfill $\blacksquare$}

\newcommand{\beq}{\begin{equation}}
\newcommand{\eeq}{\end{equation}}
\newcommand{\beqa}{\begin{eqnarray}}
\newcommand{\eeqa}{\end{eqnarray}}
\newcommand{\beqas}{\begin{eqnarray*}}
\newcommand{\eeqas}{\end{eqnarray*}}
\newcommand{\bit}{\begin{itemize}}
\newcommand{\eit}{\end{itemize}}
\newcommand{\benum}{\begin{enumerate}}
\newcommand{\eenum}{\end{enumerate}}

\newcommand{\comment}[1]{}
\newcommand{\mmp}[1]{\textcolor{darkred}{\em MMP: #1}}
\newcommand{\yuchaz}[1]{\textcolor{windowblue}{\em Yuchia: #1}}

\renewcommand{\mmp}[1]{} \renewcommand{\yuchaz}[1]{}

\newcommand{\diffmapalg}{{\sc DiffMap}}
\newcommand{\lapalg}{{\sc Laplacian}}
\newcommand{\rmetricalg}{{\sc RMetric}}

\newcommand{\coordsearch}{{\sc IndEigenSearch}}
\newcommand{\greedycoordsearch}{{\sc GreedyIndEigenSearch}}
\newcommand{\llrcoordsearch}{{\sc LLRCoordSearch}}

\newcommand{\regusearch}{{\sc ReguParamSearch}}

\newcommand{\grad}{\operatorname{grad}}

\newcommand{\rrr}{{\mathbb R}}

\newcommand{\T}{\mathcal T}
\newcommand{\M}{\mathcal{M}}

\newcommand{\id}{g}
\newcommand{\gpush}{g_{*\phi}}
\newcommand{\gpushs}{g_{*\phi_S}}
\newcommand{\g}{\gpush}

\newcommand{\loss}{{\mathfrak L}}
\newcommand{\risk}{{\mathfrak R}}
\newcommand{\Vol}{\operatorname{Vol}}
\newcommand{\diag}{\operatorname{diag}}
\newcommand{\rank}{\operatorname{rank}}
\DeclareMathOperator*{\argmax}{argmax} 

\newcommand{\volnorm}{\Vol_\mathrm{norm}}
\newcommand{\inv}[1]{\frac{1}{#1}}
\newcommand{\vect}[1]{\mathbf{\bm{#1}}}
\newcommand{\tilj}{\tilde{\jmath}}
\newcommand{\tilH}{\tilde{\vect{H}}}
\newcommand{\tilG}{\tilde{\vect{G}}}

\newcommand{\xx}{\vect{x}}
\newcommand{\yy}{\vect{y}}
\newcommand{\jacobian}{\textbf{\textsf{D}}}

\newcommand\scirc[1][.7]{\mathbin{\vcenter{\hbox{\scalebox{#1}{$\circ$}}}}}
\newcommand{\suppnumberprefix}{S}

%% file: arxiv-motivation.tex
\section{Motivation}
\label{sec:motivation}
We study a well-documented deficiency
of manifold learning algorithms. Namely, as shown in
\cite{goldberg08}, algorithms that find their output $\vect{Y}=\varphi(\vect{X})$ by
minimizing a quadratic form under some normalization constraints, fail
spectacularly when the data manifold has a large aspect ratio, that
is, it extends much more in one direction than in others, as the long,
thin strip illustrated in 
Figure \ref{fig:cont-long-strip-example}. This class, called {\em output
normalized (ON)} algorithms, includes Locally Linear Embedding
(LLE), Laplacian Eigenmap, Local Tangent Space Alignment (LTSA),
Hessian Eigenmaps (HLLE), and Diffusion maps.
The problem, often observed in practice, was formalized in
\cite{goldberg08} as failure to find an $\vect{Y}$ affinely
equivalent to $\vect{X}$. They give sufficient conditions for failure,
using a linear algebraic perspective. The conditions show that,
especially when noise is present, the problem is pervasive.

In the present paper, we revisit the problem from a {\em differential
  geometric} perspective.  First, we define failure not as distortion,
but as drop in the {\em rank} of the mapping $\phi$ represented by
the embedding algorithm. In other words, the algorithm fails when the
map $\phi$ is not invertible, or, equivalently, when the dimension
$ \dim \phi(\M)<\dim \M=d$, where $\M$ represents the idealized
data manifold, and $\dim$ denotes the intrinsic dimension. Figure
\ref{fig:cont-long-strip-example} demonstrates that the problem is
fixed by choosing the eigenvectors with care. In fact, as we show
  in Section \ref{sec:theory}, it is known \cite{bates:16} that for
  the DM and LE algorithms, under mild geometric conditions, one
  can {\em always} find a finite set of $m$ eigenfunctions that
  provide a smooth $d$-dimensional map. We call this problem the {\em
  Independent Eigencoordinate Selection} (IES) problem, formulate it
and explain its challenges in Section \ref{sec:problem-formulation}.

Our second main contribution (Section \ref{sec:algorithm}) is to design a bicriterial method that will select from a set of {\em coordinate functions} $\phi_1,\ldots \,\phi_m$, a subset $S$ of small size that provides a smooth full-dimensional embedding of the data. The IES problem requires searching over a combinatorial number of sets. We show (Section \ref{sec:algorithm}) how to drastically reduce the computational burden per set for our algorithm. 
Third, we analyze the proposed criterion under asymptotic limit (Section \ref{sec:theory}). 
Finally (Section \ref{sec:experiments}), we show examples of
successful selection on real and synthetic data. The experiments also
demonstrate that users of manifold learning for other than toy data {\em must} be aware of the IES problem and have tools for handling it.
Notations table, proofs, a library of hard examples, extra experiments and analyses are in Supplements A--G; Figure/Table/Equation references with prefix \suppnumberprefix~are in the Supplement.

%% file: arxiv-introduction.tex
\section{Background on manifold learning}
\label{sec:background-ml-alg}
{\bf Manifold learning  (ML) and intrinsic geometry~}
Suppose we observe data $\vect{X}\in\rrr^{n\times D}$, with data points denoted by $\vect{x}_i \in \rrr^D ~\forall~i\in [n]$, that are sampled
from a {\em smooth}\footnote{In this paper, a smooth function or manifold will be assumed to be of class at least ${\mathcal C}^3$.} $d$-dimensional submanifold $\M \subset \rrr^D$. Manifold Learning algorithms map $\xx_{i},i\in[n]$ to $\vect{y}_i=\phi(\xx_i)\in\rrr^s$, where $d\leq s \ll D$, thus reducing the dimension of the data $\vect{X}$ while preserving (some of) its properties.
Here we present the LE/DM algorithm, 
but our results can be applied to
other ML methods with slight modification. 

The first two steps of LE/DM \cite{coifman:06,nadler:06} algorithms are generic; 
they are performed by
most ML algorithms.
First we encode the neighborhood relations in a {\em neighborhood graph}, 
which is an undirected graph $G(V, E)$ with vertex set
$V$ be the collection of 
all points $i\in[n]$ and edge set $E$ be the collections of tuple $(i, j) \in V^2$ such 
that $i$ is $j$'s neighbor (and vice versa). Common methods for building neighborhood
graphs include $r$-radius graph and $k$-nearest neighbor graph. Readers are encouraged to 
refer to \cite{HeinAL:07,TingHJ:10} for details. In this paper, $r$-radius graph are considered,
for which principled methods to select the neighborhood size and dimension exist.
For such method, the edge set $E$ of the neighborhood graph is constructed as follow:
$E = \{(i, j) \in V^2: \|\vect{x}_i - \vect{x}_j\|_2 \leq r\}$.
Closely related to the neighborhood graph is the {\em kernel matrix}
$\vect{K} \in \rrr^{n \times n}$ whose elements are
$K_{ij} = \exp\left( -\tfrac{\|\vect{x}_i-\vect{x}_{j}\|}{\varepsilon^2}\right)  \; \text{if } (i, j)\in E \text{ and } 0 \text{ otherwise}$.
Typically, the radius $r$ and the {\em bandwidth} parameter
$\varepsilon$ are related by $r=c \varepsilon$ with $c$ a small
constant greater than 1, e.g., $c\in [3, 10]$. This ensures that $\vect{K}$ is close to its limit
when $r\rightarrow \infty$ while remaining sparse, with sparsity
structure induced by the neighborhood graph. 
Having obtained the kernel matrix $\vect{K}$,
we then construct the {\em renormalized graph Laplacian} matrix $\vect{L}$ \cite{coifman:06}, also called the 
{\em sample Laplacian}, or {\em Diffusion Maps Laplacian}, by the following:
$\vect{L} = \vect{I}_n -  \tilde{\vect{W}}^{-1}\vect{W}^{-1}\vect{K}\vect{W}^{-1},$
where $\vect{W} = \diag(\vect{K}\vect{1}_n)$ with $\vect{1}_n$ be all one vectors
and $\tilde{\vect{W}} = \diag\left(\vect{W}^{-1}\vect{K}\vect{W}^{-1}\vect{1}_n\right)$. 
The method of constructing $\vect{L}$ as described above
guarantees that if the data are sampled from a manifold $\M$, $\vect{L}$
converges to $\Delta_\M$ \cite{HeinAL:05,TingHJ:10}. A summary of the
construction of $\vect{L}$ can be found in Algorithm \ref{alg:graph-laplacian-alg} \lapalg.
The last step of LE/DM algorithm embeds the data by solving the
minimum eigen-problem of 
$\vect{L}$.
The desired $m$ dimensional embedding
coordinates are obtained from the second to $m+1$-th principal
eigenvectors of graph Laplacian $\vect{L}$, with $0=\lambda_0< \lambda_1\leq\ldots \leq\lambda_m$, i.e., $\vect{y}_i=(\phi_1(\xx_i),\ldots \phi_m(\xx_i))$ (see also Supplement \ref{sec:pseudocodes}).

To analyze ML algorithms, it is useful to consider the limit of the mapping $\phi$ when the data is the entire manifold $\M$. We denote this limit also by $\phi$, and its image by $\phi(\M)\in\rrr^m$. For standard algorithms such as LE/DM, it is known that this limit exists \cite{coifman:06,belkin:07,HeinAL:05,HeinAL:07,TingHJ:10}.
One of the fundamental requirements of ML is to preserve the neighborhood relations in the original data. In mathematical terms, we require that $\phi:\M\rightarrow\phi(\M)$ is a {\em smooth embedding}, i.e., that $\phi$ is a smooth function (i.e. does not break existing neighborhood relations) whose Jacobian $\jacobian\phi(\xx)$ is full rank $d$ at each $\xx\in\M$ (i.e. does not create new neighborhood relations). 

\paragraph{The pushforward Riemannian metric}
A smooth $\phi$ does not typically
preserve geometric quantities such as distances along curves in $\M$. These
concepts are captured by {\em Riemannian geometry}, and
 we additionally assume that $(\M,\id)$ is a {\em Riemannian
  manifold}, with the metric $\id$ induced from $\rrr^D$. One can
always associate with $\phi(\M)$ a Riemannian metric $\gpush$, called
the {\em pushforward Riemannian metric} \cite{LeeJohnM.2003Itsm}, which preserves the geometry of $(\M,\id)$; $\gpush$ is defined by
\beq \label{eq:rmetric0}
\langle \vect{u}, \vect{v}\rangle_{\g(\xx)} = \left\langle \jacobian\phi^{-1}(\xx)\vect{u}, \jacobian\phi^{-1}(\xx)\vect{v} \right\rangle_{g(\xx)} 
\text{ for all } \vect{u}, \vect{v}\in\T_{\phi(\vect{x})}\phi(\M)	
\eeq
In the above, $\T_\vect{x}\M$, $\T_{\phi(\vect{x})}\phi(\M)$ are tangent subspaces, $\jacobian\phi^{-1}(\vect{x})$ maps vectors from
$\T_{\phi(\vect{x})}\phi(\M)$ to $\T_\vect{x}\M$, and $\langle ,
\rangle$ is the Euclidean scalar product.
For each $\phi(\vect{x}_i)$, the associated push-forward Riemannian
metric expressed in the coordinates of $\rrr^m$, is a symmetric,
semi-positive definite $m\times m$ matrix $\vect{G}(i)$ of rank
$d$. The scalar product $\langle \vect{u},\vect{v} \rangle_{\g(\xx_i)}$ takes
the form $\vect{u}^\top\vect{G}(i)\vect{v}$.  Given an embedding $\vect{Y} = \phi(\vect{X})$, $\vect{G}(i)$ can be estimated  by Algorithm \ref{alg:rmetric-alg} (\rmetricalg) of \cite{2013arXiv1305.7255P}.
The \rmetricalg~algorithm also returns the {\em co-metric} $\vect{H}(i)$, which is the pseudo-inverse of the metric $\vect{G}(i)$, and its Singular Value Decomposition $\vect{\Sigma}(i),\vect{U}(i)\in\mathbb R^{m\times d}$. The latter represents an orthogonal basis of $\T_{\phi(\vect{x})}(\phi(\M))$.
\begin{algorithm}[!htb]
    \SetKwInOut{Input}{Input}
	\SetKwInOut{Output}{Return}
	\SetKwComment{Comment}{$\triangleright $\ }{}
    \Input{Embedding $\vect{Y}\in \mathbb R^{n\times m}$, Laplacian $\vect{L}$, intrinsic dimension $d$}
    \For{all $\vect{y}_i \in \vect{Y}, k = 1\to m, l = 1\to m$}{
		$[\tilde{\vect{H}}(i)]_{kl} = \sum_{j\neq i} L_{ij} (y_{jl} - y_{il})(y_{jk} - y_{ik})$
    }
    \For{$i = 1\to n$}{
		$\vect{U}(i)$, $\vect{\Sigma}(i) \gets $ {\sc ReducedRankSVD}$(\tilde{\vect{H}}(i), d)$ \\
		$\vect{H}(i) = \vect{U}(i)\vect{\Sigma}(i)\vect{U}(i)^\top$\\
    $\vect{G}(i) = \vect{U}(i)\vect{\Sigma}^{-1}(i)\vect{U}(i)^\top$ \\
    }
    \Output{$\vect{G}(i),\vect{H}(i) \in \mathbb R^{ m\times m}$, $\vect{U}(i)\in \mathbb R^{m\times d}$,  $\vect{\Sigma}(i)\in \mathbb R^{d\times d}$, for $i\in [n]$}
    \caption{\rmetricalg}
    \label{alg:rmetric-alg}
\end{algorithm}

%% file: arxiv-algorithm.tex
\section{IES problem, related work, and challenges}
\label{sec:problem-formulation}
\paragraph{An example}
\begin{wrapfigure}[12]{r}{0.5\textwidth}
    \centering
    \vspace{-22pt}
    \subfloat[][]
    {\includegraphics[width=0.48\linewidth]{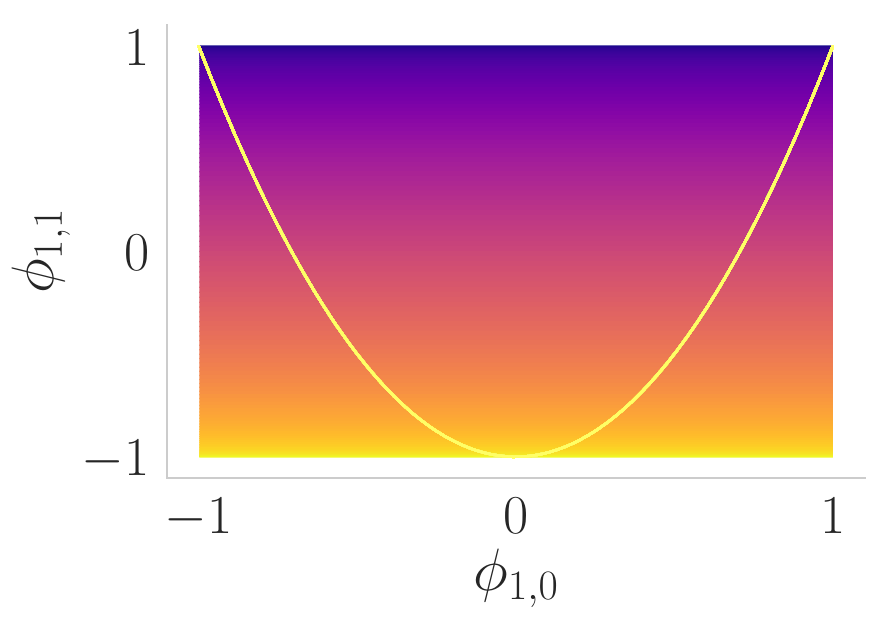}
    \label{fig:cont-long-strip-27}}\hfill
    \subfloat[][]
    {\includegraphics[width=0.48\linewidth]{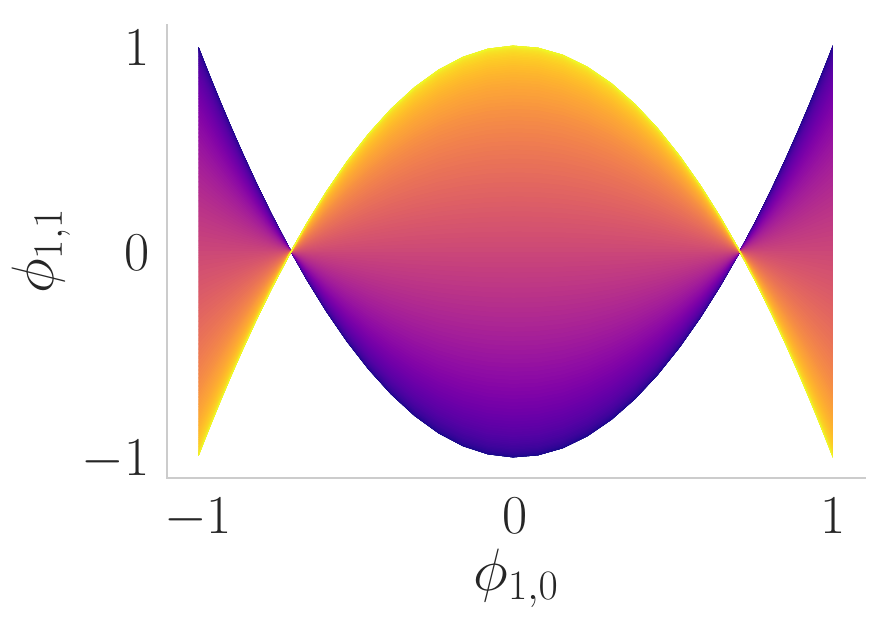}
    \label{fig:cont-long-strip-29}}\hfill   
    \caption{(a) Eigenfunction $\phi_{1, 0}$ versus ${\phi}_{2, 0}$ (curve) or ${\phi}_{0, 1}$ (two dimensional manifold). (b) Eigenfunction $\phi_{1, 0}$ versus ${\phi}_{1, 1}$. All three manifolds are colored by the parameterization $h$.}
    \label{fig:cont-long-strip-example}
\end{wrapfigure}
Consider a continuous
two dimensional strip with width $W$, height $H$, and {\em aspect
  ratio} $W/H\geq 1$, parametrized by coordinates $w\in[0,W], h\in[0,H]$. The
eigenvalues  and eigenfunctions
of the Laplace-Beltrami operator
$\Delta$ with von Neumann boundary conditions \cite{strauss2007partial} are
$
	\lambda_{k_1, k_2} = \left(\tfrac{k_1\pi}{W}\right)^2 + \left(\tfrac{k_2\pi}{H}\right)^2$, respectively 
$
	\phi_{k_1, k_2}(w,h) = \cos\left(\frac{k_1\pi w}{W}\right)\cos\left(\frac{k_2\pi h}{H}\right)
$.
Eigenfunctions $\phi_{1, 0}$, $\phi_{0, 1}$
are in bijection with the $w,h$ coordinates (and give a full rank
embedding), while the mapping by $\phi_{1, 0},\phi_{2, 0}$ provides no
extra information regarding the second dimension $h$ in the underlying
manifold (and is rank 1). 
Theoretically, one can choose as coordinates eigenfunctions indexed by
$(k_1,0),(0,k_2)$,
but, in practice, $k_1$, and $k_2$ are usually
unknown, as the eigenvalues are index by their rank $0 = \lambda_{0} <
\lambda_{1}\leq \lambda_{2} \leq\cdots$.
For a two dimensional strip,
it is known \cite{strauss2007partial} that $\lambda_{1, 0}$ always corresponds to
$\lambda_{1}$ and $\lambda_{0, 1}$ corresponds to $\lambda_{(\lceil
  W/H\rceil)}$. Therefore, when $W/H> 2$, the mapping of the strip to
$\rrr^2$ by $\phi_1,\phi_2$ is low rank, while the mapping by
$\phi_1,\phi_{\lceil W/H\rceil}$ is full rank. Note that other
mappings of rank 2 exist, e.g., $\phi_1,\phi_{\lceil W/H\rceil + 2}$
($k_1 = k_2 = 1$ in Figure \ref{fig:cont-long-strip-29}).  These
embeddings reflect progressively higher frequencies, as the corresponding
eigenvalues grow larger.

\label{sec:ies-problem}
\paragraph{Prior work} \cite{goldberg08} is the first work to give the IES problem a rigurous analysis. Their paper focuses on rectangles, and the failure illustrated in Figure \ref{fig:cont-long-strip-27} is defined as obtaining a mapping $\vect{Y} = \phi(\vect{X})$ that is not {\em affinely equivalent} with the original data. They call this the {\em Price of Normalization} and explain it in terms of the variances along $w$ and $h$.
\cite{dsilva2018parsimonious} is the first to frame the failure in terms of the rank of $\phi_S=\{\phi_k: k\in S\subseteq [m]\}$, calling it the {\em repeated eigendirection problem}. They propose a heuristic, \llrcoordsearch, based on the observation that if $\phi_k$ is a
repeated eigendirection of $\phi_1,\cdots, \phi_{k-1}$, one can fit $\phi_k$ with {\em local linear regreesion} on predictors $\phi_{[k-1]}$ with low leave-one-out errors $r_k$. 

\paragraph{Existence of solution} Before trying to find an algorithmic
solution to the IES problem, we ask the question whether this is even
possible, in the smooth manifold setting. Positive answers are given
in \cite{Portegies:16}, which proves that isometric embeddings by DM with finite
$m$ are possible, and more recently in \cite{bates:16}, which proves
that any closed, connected Riemannian manifold $\M$ can be smoothly
embedded by its Laplacian eigenfunctions $\phi_{[m]}$ into $\rrr^m$
for some $m$, which depends only on the intrinsic dimension $d$ of
$\M$, the volume of $\M$, and lower bounds for {\em injectivity
radius} and {\em Ricci curvature}. The example in Figure \ref{fig:cont-long-strip-27} 
demonstrates that, typically, not all $m$ eigenfunctions are
needed. I.e., there exists a set $S\subset [m]$, so that
$\phi_S$ is also a smooth embedding. We 
follow \cite{dsilva2018parsimonious} in calling such a set $S$ {\em
  independent}. It is not known how to find an independent $S$
analytically for a given $\M$, except in special cases such as the
strip. In this paper, we propose a {\em finite sample} and
algorithmic solution, and we support it with asymptotic theoretical
analysis.

\paragraph{The IES Problem} We are given data $\vect{X}$, and the output of
an embedding algorithm (DM for simplicity) $\vect{Y}=\phi(\vect{X})
= [\phi_1,\cdots,\phi_m] \in\rrr^{n\times m}$. We assume that
$\vect{X}$ is sampled from a $d$-dimensional manifold $\M$, with known
$d$, and that $m$ is sufficiently large so that $\phi(\M)$ is a
smooth embedding. Further, we assume that there is a set $S\subseteq
[m]$, with $|S|=s\leq m$, so that $\phi_S$ is also a smooth
  embedding of $\M$.
We propose to find such set $S$ so that the rank of $\phi_S$ is $d$ on $\M$ and $\phi_S$ varies as slowly as possible.

\paragraph{Challenges}
(1) Numerically, and on a finite sample, distiguishing between a full rank mapping and a rank-defective one is imprecise. Therefore, we substitute for rank the volume of a unit parallelogram in $\T_{\phi(\xx_i)}\phi(\M)$. (2) Since $\phi$ is {\em not} an isometry, we must separate the local distortions introduced by $\phi$ from the estimated rank of $\phi$ at $\xx$.
(3) Finding the optimal balance between the above desired properties.
(4) In \cite{bates:16} it is strongly suggested that $s$ the number of eigenfunctions needed may exceed the {\em Whitney embedding dimension} ($\leq 2d$) \cite{LeeJohnM.2003Itsm}, and that this number may depend on injectivity radius, aspect ratio, and so on.  Section \ref{sec:hist-heuristic-check-s} shows an example of a flat 2-manifold, the {\em strip with cavity}, for which $s>2$. In this paper, we assume that $s$ and $m$ are given and focus on selecting $S$ with $|S|=s$ unless otherwise stated; for completeness, in Section \ref{sec:hist-heuristic-check-s} we present a heuristic to select $s$. 

\paragraph{(Global) functional dependencies, knots and crossings} Before we proceed, we describe three different ways a mapping $\phi(\M)$ can fail to be invertible. The first, {\em (global) functional dependency} is the case when $\rank \jacobian \phi<d$ on an open subset of $\M$, or on all of $\M$ (yellow curve in Figure \ref{fig:cont-long-strip-27}); this is the case most widely recognized in the literature (e.g., \cite{goldberg08,dsilva2018parsimonious}). The {\em knot} is the case when $\rank \jacobian\phi<d$ at an isolated point (Figure \ref{fig:cont-long-strip-29}). Third, the {\em crossing} (Figure \ref{fig:d8-emb-rank-2} in Supplement \ref{sec:addi-exps}) is the case when $\phi:\M\rightarrow \phi(\M)$ is not invertible at $\xx$, but $\M$ can be covered with open sets $U$ such that the restriction $\phi:U\rightarrow \phi(U)$ has full rank $d$. Combinations of these three exemplary cases can occur. The criteria and approach we define are based on the (surrogate) rank of $\phi$, therefore they will not rule out all crossings. We leave the problem of crossings in manifold embeddings to future work, as we believe that it requires an entirely separate approach (based, e.g., or the injectivity radius or density in the co-tangent bundle rather than differential structure). 
\section{Criteria and algorithm}
\label{sec:algorithm}
\subsection{A geometric criterion}
  We start with the main idea in evaluating the quality of a subset
  $S$ of coordinate functions. At each data point $i$, we consider the
  orthogonal basis $\vect{U}(i)\in\mathbb R^{m\times d}$ of the $d$
  dimensional tangent subspace $\T_{\phi(\vect{x}_i)}\phi(\M)$. The projection of the columns of $\vect{U}(i)$ onto
  the subspace $\T_{\phi(\vect{x}_i)}\phi_S(\M)$ is $\vect{U}(i)[S,
    :]\equiv \vect{U}_S(i)$.  The following Lemma connects
  $\vect{U}_S(i)$ and the co-metric $\vect{H}_S(i)$ defined by
  $\phi_S$, with the {\em full} $\vect{H}(i)$.
\begin{lemma}
Let $\vect{H}(i) = \vect{U}(i)\vect{\Sigma}(i)\vect{U}(i)^\top $ be the co-metric defined by embedding $\phi$, $S\subseteq [m]$, $\vect{H}_S(i) $ and $\vect{U}_S(i)$ defined above.  Then $\vect{H}_S(i) = \vect{U}_S(i)\vect{\Sigma}(i)\vect{U}_S(i)^\top= \vect{H}(i)[S, S]$.
\label{lm:hs-is-submatrix-of-h}
\end{lemma}
The proof is straightforward and left to the reader. Note that Lemma \ref{lm:hs-is-submatrix-of-h} is responsible for the efficiency of the search over sets $S$, given that the push-forward co-metric $\vect{H}_S$ can be readily obtained as a submatrix of $\vect{H}$.
Denote by $\vect{u}^S_k(i)$ the $k$-th column of $\vect{U}_S(i)$. 
We further normalize each  $\vect{u}^S_k$ to length 1 and define the {\em normalized projected volume} 
$\volnorm(S,i) = \frac{\sqrt{\det(\vect{U}_S(i)^\top\vect{U}_S(i))}}{\prod_{k=1}^d \|\vect{u}^S_k(i)\|_2}$.
Conceptually, $\volnorm(S,i)$ is the volume spanned by a (non-orthonormal) ``basis'' of unit vectors in $\T_{\phi_S(\xx_i)}\phi_S(\M)$; $\volnorm(S,i)=1$ when $\vect{U}_S(i)$ is orthogonal, and it is 0 when $\rank \vect{H}_S(i)<d$. In Figure \ref{fig:cont-long-strip-27}, the $\volnorm(\{1, 2\})$ with $\phi_{\{1,2\}} = \{{\phi}_{1, 0}, {\phi}_{2, 0}\}$ is close to zero, since the projection of the two tangent vectors is parallel to the yellow curve; however $\volnorm(\{1, \lceil w/h\rceil\},i)$ is almost 1, because the projections of the tangent 
vectors $\vect{U}(i)$ will be (approximately) orthogonal. Hence, $\volnorm(S,i)$ away from 0 indicates a non-singular $\phi_S$ at $i$, and we use the average $\log \volnorm(S,i)$, which penalizes values near 0 highly, as the {\em rank quality} $\risk(S)$ of $S$. 

Higher frequency $\phi_S$ maps with high $\risk(S)$ may exist, being
either smooth, such as the embeddings of the strip mentioned
previously, or containing knots involving only small fraction of
points, such as $\phi_{\phi_{1, 0}, \phi_{1, 1}}$ in Figure
\ref{fig:cont-long-strip-27}.
To choose the lowest frequency, slowest varying smooth map, a
regularization term consisting of the eigenvalues $\lambda_k$, $k\in S$,
of the graph Laplacian $\vect{L}$ is added, obtaining the criterion
\begin{equation}
	\loss(S; \zeta) = \underbrace{\inv{n}\sum_{i=1}^n \log\sqrt{\det\left({\vect{U}}_S (i)^\top {\vect{U}}_S (i)\right)}}_{\risk_1(S) = \inv{n}\sum_{i=1}^n \risk_1(S; i)} - \underbrace{\inv{n}\sum_{i=1}^n\sum_{k=1}^d \log\|{\vect{u}}^S_k(i)\|_2}_{\risk_2(S) = \inv{n}\sum_{i=1}^n \risk_2(S; i)} - \zeta\sum_{k\in S} \lambda_k
	\label{eq:loss-func}
\end{equation}
\subsection{Search algorithm}
\begin{wrapfigure}[20]{r}{0.56\textwidth}
\vspace{-15pt}
\IncMargin{1.6em}
\begin{algorithm}[H]
    \setstretch{1.15}
    \SetKwInOut{Input}{Input}
    \SetKwInOut{Output}{Return}
    \SetKwComment{Comment}{$\triangleright $\ }{}
    \Indm
    \Input{Data $\vect{X}$, bandwith $\varepsilon$, intrinsic dimension $d$, embedding dimension $s$, regularizer $\zeta$}
    \Indp
    $\vect{Y}\in\mathbb R^{n\times m}, \vect{L}, \vect{\lambda}\in\mathbb R^m\gets$\diffmapalg$(\vect{X}, \varepsilon)$ \\
    $\vect{U}(i),\cdots,\vect{U}(n) \gets$\rmetricalg$(\vect{Y}, \vect{L}, d)$ \\
    \For{$S \in \{S' \subseteq [m]: |S'| = s, 1\in S'\}$}{
        $\risk_1(S) \gets 0; \risk_2(S) \gets 0$ \\
        \For{$i = 1,\cdots, n$}{
            $\vect{U}_S(i) \gets \vect{U}(i)[S, :]$ \\
            $\risk_1(S) \mathrel{+}= \inv{2n}\cdot \log\det\left(\vect{U}_S(i)^\top\vect{U}_S(i)\right)$\\
            $\risk_2(S) \mathrel{+}= \inv{n}\cdot \sum_{k=1}^d\log\|u^S_k(i)\|_2$
        }
        $\loss(S; \zeta) = \risk_1(S) - \risk_2(S) - \zeta\sum_{k\in S}\lambda_k$
    }
    $S_* = \argmax_S \loss(S;\zeta)$ \\
    \Indm
    \Output{Independent eigencoordinates set $S_*$}
    \Indp
    \caption{\coordsearch}
    \label{alg:indep-coord-search}
\end{algorithm}
\DecMargin{1.6em}
\end{wrapfigure}
With this criterion, the IES problem  turns into a subset
selection problem parametrized by $\zeta$
\begin{equation}
    S_*(\zeta) = \argmax_{S\subseteq[m]; |S|=s; 1\in S} \loss(S; \zeta)
    \label{eq:subset-selection-problem}
\end{equation}

Note that we force the first coordinate $\phi_1$ to always be chosen, since this coordinate cannot be functionally dependent on previous ones, and, in the case of DM, it also has lowest frequency.
Note also that $\risk_1$ and $\risk_2$ are both submodular set
function (proof in Supplement \ref{sec:submodularity-loss-func}). 
For large
$s$ and $d$, algorithms for optimizing over the difference of submodular functions can be 
used (e.g., see \cite{Iyer:2012:AAM:3020652.3020697}). For the experiments in this paper, 
we have $m = 20$ and $d, s = 2\sim 4$, which enables us to use exhaustive search
to handle \eqref{eq:subset-selection-problem}. 
The exact search algorithm is summarized in Algorithm
\ref{alg:indep-coord-search} \coordsearch.
A greedy variant is also proposed and analyzed in Supplement \ref{sec:greedy-algorithm}.

\subsection{Regularization path and choosing $\zeta$}
According to \eqref{eq:loss-func}, the optimal subset $S_*$
depends on the parameter $\zeta$. 
The regularization path  $\ell(\zeta) = \max_{S\subseteq [m]; |S|=s; 1\in S} \loss(S; \zeta)$ is the upper envelope of multiple lines (each correspond to a set $S$) with  slopes $-\sum_{k\in S}\lambda_k$ and intercepts $\risk(S)$. The larger $\zeta$ is, the more the 
 lower frequency subset penalty prevails, and for sufficiently large $\zeta$ the algorithm will output $[s]$. 
In the supervised learning framework, the regularization parameters are often chosen
by cross validation. Here we propose a second criterion, that effectively limits how much $\risk(S)$ may be ignored, or alternatively, bounds $\zeta$ by a data dependent quantity.
Define the {\em leave-one-out regret} of point $i$ as follows
\begin{equation}
	\mathfrak D(S, i) = \mathfrak R(S_*^i;  [n]\backslash \{i\}) - \mathfrak R(S; [n]\backslash \{i\}) \text{ with } S_*^i = \mathrm{argmax}_{S\subseteq [m]; |S|=s; 1\in S} \risk(S;i)
	\label{eq:adversarial-gain}
\end{equation}
In the above, we denote $\mathfrak R(S; T) = \inv{|T|} \sum_{i\in T} \risk_1(S; i) -\risk_2(S; i)$ for some subset $T\subseteq [n]$.
The quantity $\mathfrak D(S, i)$ in 
\eqref{eq:adversarial-gain} measures the gain in $\risk$ 
if all the other points $[n]\backslash \{i\}$ choose the optimal subset $S_*^i$.
If the regret $\mathfrak D(S, i)$ is larger than zero, it indicates
that the alternative choice might be better compared to original
choice $S$.  Note that the mean value for all $i$, i.e.,
$\inv{n}\sum_i\mathfrak D(S, i)$ depends also on the variability of
the optimal choice of points $i$, $S_*^i$.  Therefore, it might not
favor an $S$, if $S$ is optimal for every $i\in [n]$.  Instead, we
propose to inspect the distribution of $\mathfrak D(S,i)$, and remove
the sets $S$ for which $\alpha$'s percentile are larger than zero,
e.g., $\alpha = 75\%$, recursively from $\zeta=\infty$ in decreasing
order.
Namely, the chosen set is $S_* = S_*(\zeta')$ with $\zeta' = \max_{\zeta \geq 0} \textsc{Percentile}(\{\mathfrak D(S_*(\zeta), i)\}_{i=1}^n, \alpha) \leq 0$. 
The optimal $\zeta_*$ value is simply chosen to be the midpoint of all the $\zeta$'s that
outputs set $S_*$
i.e.,
$\zeta_* = \inv{2}\left(\zeta' + \zeta''\right)$, where
$\zeta'' = \min_{\zeta \geq 0} S_*(\zeta) = S_*(\zeta')$.
The procedure \regusearch~is summarized in Algorithm \ref{alg:choose-zeta}.

\begin{algorithm}[!htb]
    \setstretch{1.15}
    \SetKwInOut{Input}{Input}
    \SetKwInOut{Output}{Return}
    \SetKwComment{Comment}{$\triangleright $\ }{}
    \Input{Threshold parameter $\alpha$}
    \For{$\zeta = \zeta_\mathrm{\max} \to 0$}{
    	\Comment{$\zeta_\mathrm{\max}$ should be sufficiently large such that $S_*(\zeta_{\max}) = [s]$}
    	$S \gets S_*(\zeta)$; $S_*\gets $\texttt{NULL}; $\zeta''\gets$ \texttt{NULL} \\
        \For{$i\in [n]$}{
        	$\mathfrak D(S, i)\gets \risk(S^i_*; [n]\backslash \{i\}) - \risk(S; [n]\backslash \{i\})$ from equation \eqref{eq:adversarial-gain} \\
    	}
    	\uIf{$\textsc{Percentile}(\{\mathfrak D(S, i)\}_{i=1}^n, \alpha) \leq 0$ {\bf and}
    		$S_* = $ \texttt{NULL} }{
    		Optimal set $S_*\gets S$ \\
    		$\zeta' \gets \zeta$~~
			\Comment{First found a set that satisfies the criterion.}
    	}
    	\uElseIf{$S_*\neq$ \texttt{NULL} {\bf and} $S_* = S_*(\zeta)$}{
    		$\zeta''\gets \zeta$~~
			\Comment{Searching for $\zeta''$}
    	}
    	\uElseIf{$S_*\neq$ \texttt{NULL} {\bf and} 
    			 $\zeta''\neq$ \texttt{NULL} {\bf and } 
    			 $S_*\neq S_*(\zeta)$}{
    		$\zeta_* \gets \inv{2}(\zeta' + \zeta'')$ \\
    		{\bf break} ~~ \Comment{Leave the loop when found $\zeta'' = \min_{\zeta\geq 0} S_*(\zeta') = S_*(\zeta)$}
    	}
    	\Else{{\bf continue}}
    }
    \Output{Optimal set $S_*$, optimal regularization parameter $\zeta_*$}
    \caption{\regusearch}
    \label{alg:choose-zeta}
\end{algorithm}

%% file: arxiv-embedding-dim-check.tex
\section{A heuristic to determine whether $s$ is sufficiently large}
\label{sec:hist-heuristic-check-s}
\begin{figure*}[t]
\centering

\subfloat[][]
{\includegraphics[width=0.24\linewidth]{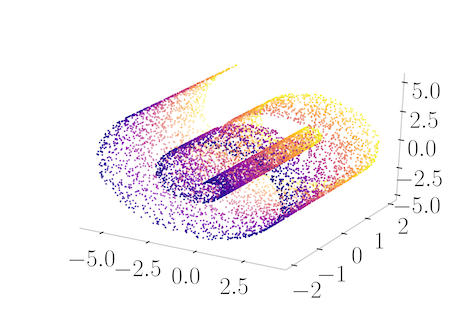}
\label{fig:d4-orig-data}}\hfill
\subfloat[][]
{\includegraphics[width=0.24\linewidth]{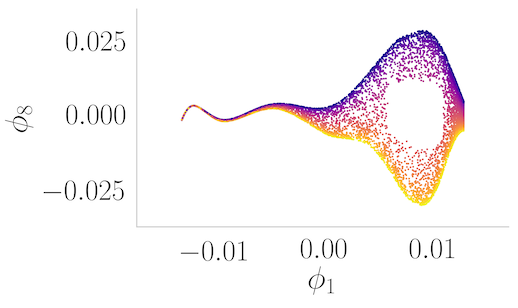}
\label{fig:d4-emb-rank-1}}\hfill
\subfloat[][]
{\includegraphics[width=0.24\linewidth]{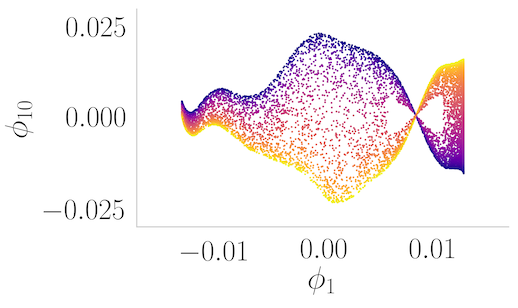}
\label{fig:d4-emb-rank-2}}\hfill
\subfloat[][]
{\includegraphics[width=0.24\linewidth]{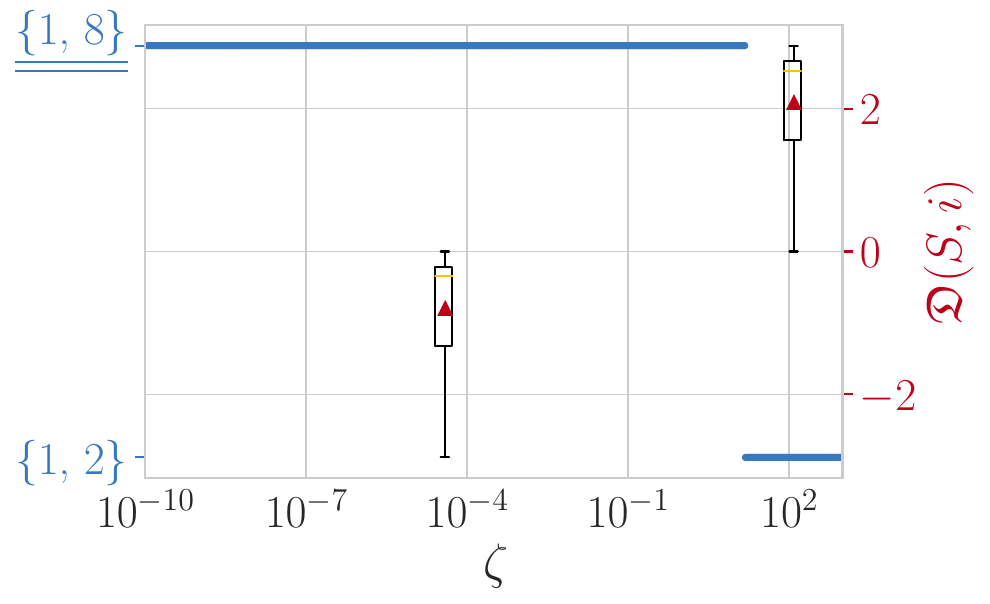}
\label{fig:d4-regu-path}}\hfill

\subfloat[][]
{\includegraphics[width=0.32\linewidth]{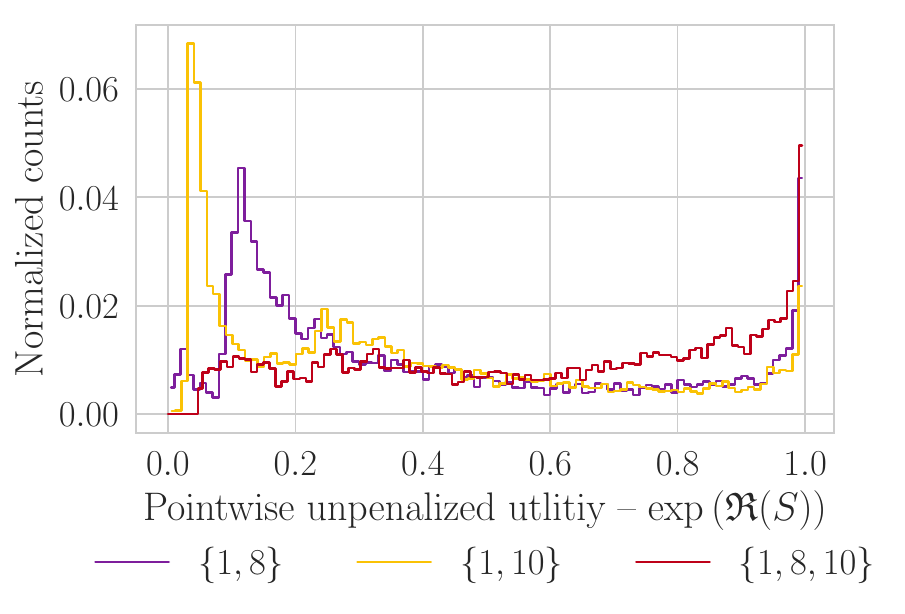}
\label{fig:d4-histogram-risk}}\hfill
\subfloat[][]
{\includegraphics[width=0.32\linewidth]{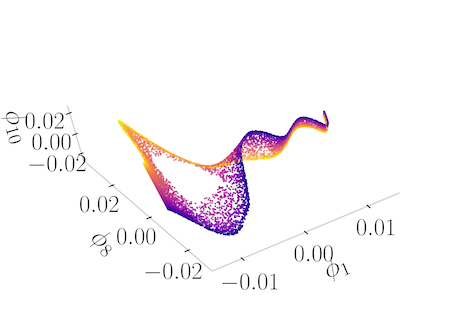}
\label{fig:d4-emb-1-8-10}}\hfill
\subfloat[][]
{\includegraphics[width=0.32\linewidth]{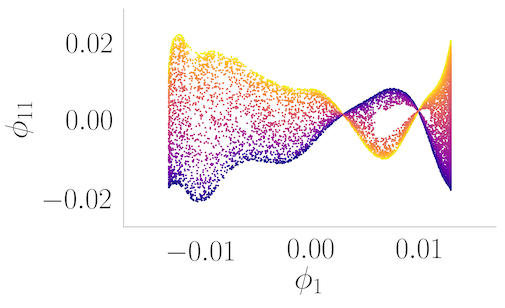}
\label{fig:d4-emb-rank-3}}\hfill
\caption{(a) Original data of $\mathcal D_4$, {\em swiss roll with hole} dataset. Embeddings with coordinate subset to be (b) $S = \{1, 8\}$, (c) $S = \{1, 10\}$, (f) $S = \{1, 8, 10\}$ and (g) $S = \{1, 11\}$ on $\mathcal D_4$. (e) Histogram of point-wise normalized projected volume on $\mathcal D_4$ for top two ranking of subsets (purple and yellow) and the union of two sets (red) obtain from \coordsearch~algorithm.}
\label{fig:d4-embedding-otherview}	
\end{figure*}

In this section, we propose a heuristic method to determine whether the given $s$ is large enough. 
Our method is based on the  histogram of  $\Vol_{\rm norm}(S,i)=\exp\left(\risk(S, i)\right)$,  the {\em normalized projected volume} of each point $i$. Recall that this volume is 
bounded between 0 and 1. Ideally, a perfect choice of cardinality $|S|$ will result in a 
concentration of mass in larger $\Vol_{\rm norm}$ region. The heuristic works as follow: at first
we check
the histogram of unpenalized $\Vol_{\rm norm}$ on the top few ranked subsets in terms of $\loss$.
If spikes in the small $\Vol_{\rm norm}$ regions are witnessed in the histogram, 
taking the union of the 
subsets and inspecting the histogram of unpenalized $\Vol_{\rm norm}$ on the combined set again.
If spikes in small $\Vol_{\rm norm}$ region diminished, one can conclude that a larger
cardinality size $|S|$ is needed for such manifold. 

We illustrate the idea on {\em swiss roll with hole} dataset in Figure \ref{fig:d4-orig-data}. 
Figure \ref{fig:d4-emb-rank-1} is the optimal subset of coordinates $S_* = \{1, 8\}$ selected
by the
proposed algorithm that best parameterize the underlying manifold. 
Figure \ref{fig:d4-regu-path} suggests one should eliminate $S_0 = \{1, 2\}$ 
because $\mathfrak D(S_0, i)\geq 0$ for all the points by \regusearch.
However, as shown in Figure \ref{fig:d4-emb-rank-1}, though it has low frequency and having
rank $2$ for most of the places, set $\{1, 8\}$ might not be suitable for data analysis for the very thin 
arms in left side of the embedding. 
Figure \ref{fig:d4-histogram-risk} is the histograms of the point-wise unpenalized $\Vol_{\rm norm}$
on different subsets. Purple and yellow curves correspond to the histogram of top two ranked 
subsets $S$ from \coordsearch. Both curves show a concentration of masses 
in small $\Vol_{\rm norm}$ region. The histogram of point-wise unpenalized $\Vol_{\rm norm}$
on $\{1, 8, 10\}$ (red curve), which is the union of the aforementioned two subsets, 
shows less concentration in the
small $\Vol_{\rm norm}$ region and implies that $|S| = 3$ might be a better choice for data analysis.
Figure \ref{fig:d4-emb-1-8-10} shows the embedding with coordinate $S = \{1, 8, 10\}$, 
which represents a two dimensional strip embedded in three dimensional space. 
The thin arc in Figure \ref{fig:d4-emb-rank-1} turns out to be a collapsed two 
dimensional manifold via projection, as shown in the upper right part of Figure 
\ref{fig:d4-emb-1-8-10} and left part of Figure \ref{fig:d4-emb-rank-2}. 
Here we have to restate that the embedding in Figure \ref{fig:d4-emb-rank-1}, although is a 
{\em degenerated} embedding, is still the best set one can choose for $s = 2$ such that 
the embedding varies slowest and has rank $2$. However, choosing $s=3$ might be better for 
data analysis.

%% file: arxiv-theory.tex
\section{$\risk$ as Kullbach-Leibler divergence}
\label{sec:theory}
In this section we analyze $\risk$ in its population version, and show that it is reminiscent of a Kullbach-Leibler divergence between {\em unnormalized} measures on $\phi_S(\M)$. The population version of the regularization term takes the form of a well-known {\em smoothness} penalty on the embedding coordinates $\phi_S$. 

\paragraph{Volume element and the Riemannian metric}
Consider a Riemannian manifold $(\M,g)$ mapped by a smooth embedding $\phi_S$ into $(\phi_S(\M),\gpushs)$, $\phi_S:\M\rightarrow \rrr^s$, where $\gpushs$ is the {\em push-forward} metric defined in \eqref{eq:rmetric0}. A Riemannian metric $g$ induces a {\em Riemannian measure} on $\M$, with volume element $\sqrt{\det g}$. 
Denote now by $\mu_{\M}$, respectively $\mu_{\phi_S(\M)}$ the Riemannian measures corresponding to the metrics induced on $\M,\phi_S(\M)$ by the ambient spaces $\rrr^D,\rrr^s$; let $g$ be the former metric.  

\begin{lemma} Let $S, \phi, \phi_S, \vect{H}_S(\xx),\vect{U}_S(\xx),\vect{\Sigma}(\xx)$ be defined as in Section \ref{sec:algorithm} and Lemma \ref{lm:hs-is-submatrix-of-h}. For simplicity, we denote by $\vect{H}_S(\vect{y})\equiv \vect{H}_S(\phi_S^{-1}(\vect{y}))$, and similarly for $\vect{U}_S(\vect{y}),\vect{\Sigma}(\vect{y})$. Assume that $\phi_S$ is a smooth embedding. Then, for any measurable function $f:\M\rightarrow \rrr$,
\beq
\int_{\M} f(\xx)d\mu_{\M}(\xx)
=\!\int_{\phi_S(\M)}\!\!\!f(\phi_S^{-1}(\vect{y}))j_S(y)d\mu_{\phi_S(\M)}(\vect{y}),
\eeq
with 
\beq\label{eq:js}
j_S(\vect{y})\;=\;1/\Vol(\vect{U}_S(\vect{y})\vect{\Sigma}_S^{1/2}(\vect{y})).
  \eeq
\label{lm:change-of-variable}
\end{lemma}
\begin{proof}
Let $\mu^*_{\phi_S(\M)}$ denote the Riemannian measure induced by $\gpushs$.
Since $(\M,g)$ and $(\phi_S(\M),\gpushs)$ are isometric by definition, $\int_{\M} f(\xx)d\mu_{\M}(\xx)=\int_{\phi_S(\M)}\!f(\phi_S^{-1}(\yy))d\mu^*_{\phi_S(\M)}(\yy)=\int_{\phi_S(\M)}$$f(\phi_S^{-1}(\yy))$$\sqrt{\det \gpushs(\yy)}d\mu_{\phi_S(\M)}(\yy)$ follows from the change of variable formula. It remains to find the expression of $j_S(\yy)=\sqrt{\det \gpushs(\yy)}$.
The matrix $\vect{U}_S(\yy)$ (note that $\vect{U}_S(\yy)$ is {\em not orthogonal}) can be written as
\beq \label{eq:uqr}
\vect{U}_S(\yy)\;=\;\vect{V}\vect{Q}_S(\yy)
\eeq
where $\vect{V}\in\rrr^{s\times d}$ is an orthogonal matrix and $\vect{Q}_S(\yy)\in\rrr^{d\times d}$ is upper triangular. Then,
\beq
\vect{H}_S(\yy)=\vect{U}_S(\yy)\vect{\Sigma}(\yy) \vect{U}_S(y)^\top=\vect{V}_S(y)\underbrace{(\vect{Q}_S(\yy)\vect{\Sigma}(\yy) \vect{Q}_S(\yy)^\top)}_{\tilH_S(\yy)}\vect{V}_S(\yy)^\top.
\eeq
In the above $\tilH_S(y)$ is the co-metric expressed in the new coordinate system induced by $\vect{V}_S(\yy)$. Hence, in the same basis, $\gpushs$ is expressed by 
\beq \label{eq:gpush-V}
\tilG_S(y)\,=\,\tilH_S(y)^{-1}\,=\;(\vect{Q}_S(\yy)\vect{\Sigma}(\yy) \vect{Q}_S(\yy)^\top)^{-1}.
\eeq
The volume element, which is invariant to the chosen coordinate system, is
\beq
\det\left(\vect{Q}_S(\yy)\vect{\Sigma}(\yy) \vect{Q}_S(\yy)^\top\right)^{-1/2}\;=\;\prod_{k=1}^d \sigma_k(\yy)^{-1/2}q_{S,kk}(\yy)^{-1}.
\eeq
From \eqref{eq:uqr}, it follows also that 
\beq \label{eq:volelem}
\det\left(\vect{Q}_S(\yy)\vect{\Sigma}(\yy) \vect{Q}_S(\yy)^\top\right)^{-1/2}=1/\Vol\left(\vect{U}_S(\yy)\vect{\Sigma(\yy)}^{1/2}\right)
\eeq
\end{proof}

\paragraph{Asymptotic limit of $\risk$}
We now study the first term of our criterion in the limit of infinite sample size. We make the following assumptions.
\begin{assu} \label{assu:S,compact} The manifold $\M$ is compact of class $\mathcal{C}^3$, and there exists a set $S$, with $|S|=s$ so that $\phi_S$ is a smooth embedding of $\M$ in $\rrr^s$.
\end{assu}
\begin{assu} \label{assu:p} The data are sampled from a distribution on $\M$ continuous with respect to $\mu_\M$, whose density is denoted by $p$. 
\end{assu}
\begin{assu} \label{assu:H}
The estimate of $\vect{H}_{S}$ in Algorithm \ref{alg:rmetric-alg} computed w.r.t. the embedding $\phi_{S}$ is consistent. 
\end{assu}
We know from \cite{bates:16}
that Assumption \ref{assu:S,compact} is satisfied for the DM/LE embedding. The remaining  assumptions are minimal requirements ensuring that limits of our quantities exist.  
Now consider the setting in Sections \ref{sec:ies-problem}, in which we have a
larger set of eigenfunctions, $\phi_{[m]}$ so that $[m]$ contains the
set $S$ of Assumption \ref{assu:S,compact}. Denote
by $\tilj_S(\vect{y})=\prod_{k=1}^d\left(||u_k^S(\vect{y})||\sigma_k(\vect{y}))^{1/2}\right)^{-1}$
a new volume element.

\begin{theorem}[Limit of $\risk$]\label{thm:risk-cont} Under Assumptions \ref{assu:S,compact}--\ref{assu:H},
  \beq
 \lim_{n\rightarrow\infty}\frac{1}{n}\sum_i\ln \risk(S,\xx_i)=\risk(S,\M),
  \eeq
and 
  \beq\label{eq:kl}
  \risk(S,\M)\,=\,-\int_{\phi_S(\M)} \ln \frac{j_S(\vect{y})}{\tilj_S(\vect{y})}p(\phi_S^{-1}(\vect{y}))j_S(\vect{y})d\mu_{\phi_S(\M)}(\vect{y})
\,\stackrel{\textit{def}}{=}\,-D(pj_S\|p\tilj_S)
\eeq
\end{theorem}
\begin{proof}
Because $\phi_S$ is a smooth embedding, $j_S(\yy)>0$ on $\phi_S(\M)$, and because $\M$ is compact, $\min_{\phi_S(\M)}j_S(\yy)>0$. Similarly, noting that $\tilj_S(\yy)\geq \prod_{k=1}^d\sigma_k^{-1/2}(\yy)$, we conclude that $\tilj_S(\yy)$ is also bounded away from 0 on $\M$.  Therefore $\ln j_S(\yy)$ and $\ln \tilj_S(\yy)$ are bounded, and the integral in the r.h.s. of \eqref{eq:kl} exists and has a finite value. Now,
\beq
\frac{1}{n}\sum_i\ln \risk(S,\xx_i)\,\rightarrow\,\int_\M \ln\risk(S,\xx)p(\xx)d\mu_\M(\xx)\,=\,\risk(S,\M).
\eeq
\beqa
&&\int_\M \ln\risk(S,\xx)p(\xx)d\mu_\M(\xx) \\
& = &\int_{\phi_S(\M)} \ln \risk(\phi_S^{-1}(\yy))p(\phi_S^{-1}(\yy))j_S(\yy)d\mu_{\phi_S(\M)}(\yy)\nonumber\\
& = &\!\!\!\int_{\phi_S(\M)}\!\! \left[\frac{1}{2}\ln \frac{\Vol\left(\vect{U}^\top_S(\yy)\vect{U}_S(\yy)\right)}{\tilj_S(\yy)}
-\frac{p(\phi_S^{-1}(\yy)\prod_{k=1}^d\sigma_k^{1/2}(\yy)}{p(\phi_S^{-1}(\yy)\prod_{k=1}^d\sigma_k^{1/2}(\yy)}\right]
  p(\phi_S^{-1}(\yy))j_S(\yy)d\mu_{\phi_S(\M)}(\yy)\nonumber \\
& = &\int_{\phi_S(\M)} \ln \frac{j_S(\yy)p(\phi_S^{-1}(\yy)}{\tilj_S(\yy)p(\phi_S^{-1}(\yy)}p(\phi_S^{-1}(\yy))j_S(\yy)d\mu_{\phi_S(\M)}(\yy) 
\;=\;-D(pj_S\|p\tilj_S)
\eeqa
\end{proof}

Note that $D(\cdot\|\cdot)$ is a Kullbach-Leibler
divergence, where the measures defined by $pj_S,p\tilj_S$ normalize to
different values; because $j_S\geq \tilj_S$ the divergence $D$ is always positive.

It is known that $\lambda_k$, the $k$-th eigenvalue of the Laplacian, converges under certain technical conditions \cite{belkin:07} to an eigenvalue of the Laplace-Beltrami operator $\Delta_\M$ and that
\beq \label{eq:lamlim}
\lambda_k(\Delta_\M)
=\langle \phi_k, \Delta_\M\phi_k\rangle
=\int_\M\|\grad \phi_k(\xx)\|^2_2d\mu(\M).
\eeq
Hence, a smaller value for the regularization term encourages the use of slow varying coordinate functions, as measured by the squared norm of their gradients, as in equation \eqref{eq:lamlim}. Hence, under Assumptions \ref{assu:S,compact}, \ref{assu:p}, \ref{assu:H}, $\loss$ converges to 
\beq
\loss(S,\M)\;=\;-D(pj_S\|p\tilj_S)-\frac{\zeta}{\lambda_1(\M)}\sum_{k\in S}\lambda_k(\M).
\eeq
The rescaling of $\zeta$ in comparison with equation \eqref{eq:loss-func} aims to make $\zeta$ adimensional, whereas the eigenvalues scale with the volume of $\M$.

%% file: arxiv-experiments.tex
\section{Experiments}
\label{sec:experiments}
\begin{figure*}[t]
\centering
\begin{tabular}{cccc}
~& Original data $\vect{X}$ & Embedding ${\phi}_{S_*}$ & Regularization path \\
\raisebox{1.4cm}[0pt][0pt]{\rotatebox[origin=c]{90}{$\mathcal D_1$}} &
\includegraphics[width=0.24\textwidth]{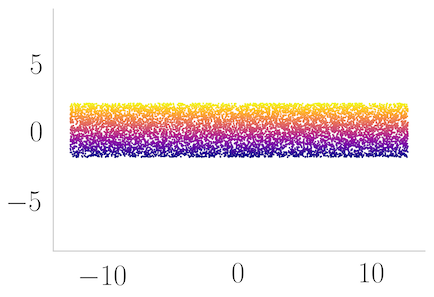}
\label{fig:d1-orig-data} &
\includegraphics[width=0.24\textwidth]{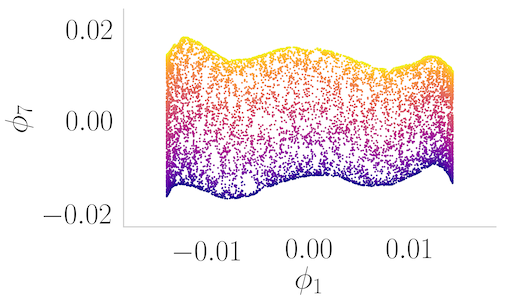}
\label{fig:d1-emb-rank-1} &
\includegraphics[width=0.24\textwidth]{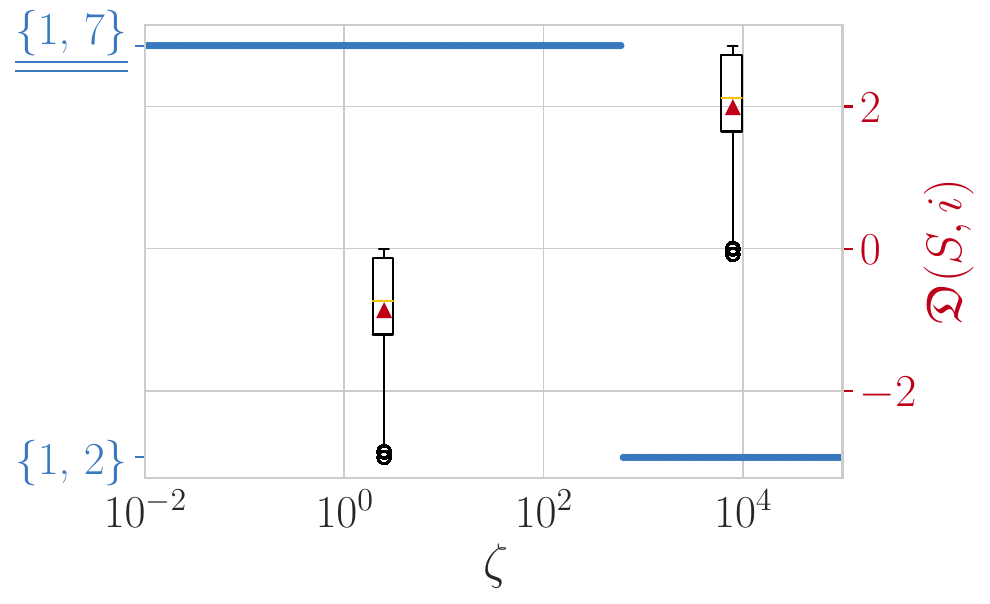}
\label{fig:d1-regu-path} \\

\raisebox{1.4cm}[0pt][0pt]{\rotatebox[origin=c]{90}{$\mathcal D_7$}} &
\includegraphics[width=0.24\textwidth]{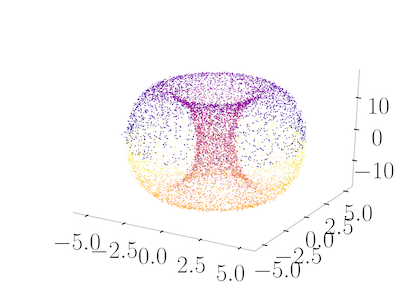}
\label{fig:d7-orig-data} & 
\includegraphics[width=0.24\textwidth]{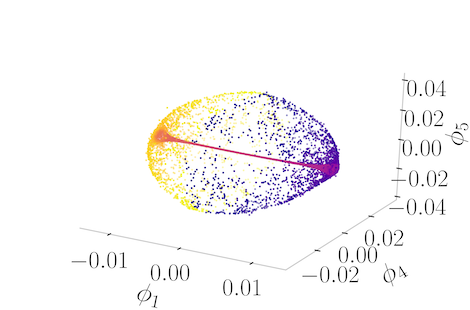}
\label{fig:d7-emb-rank-1} &
\includegraphics[width=0.24\textwidth]{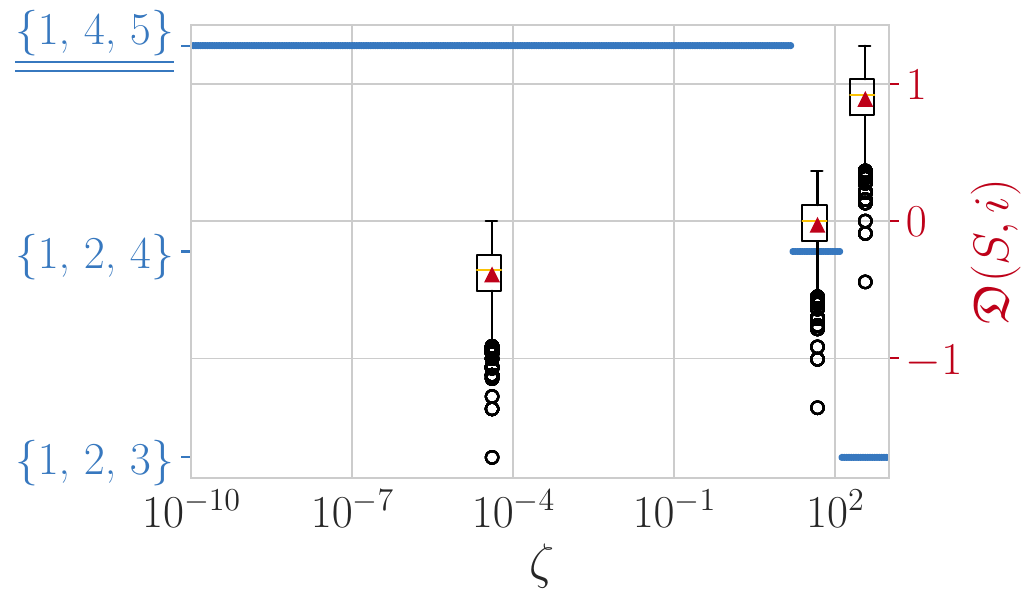}
\label{fig:d7-regu-path} \\

\raisebox{1.4cm}[0pt][0pt]{\rotatebox[origin=c]{90}{$\mathcal D_{13}$}} &
\includegraphics[width=0.24\textwidth]{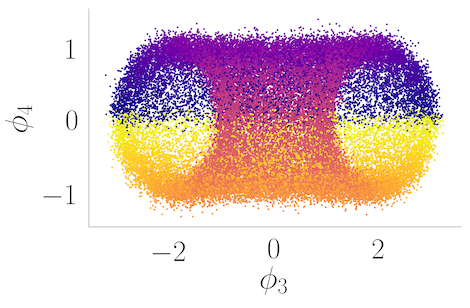}
\label{fig:d14-orig-data-axis-3-4} &
\includegraphics[width=0.24\textwidth]{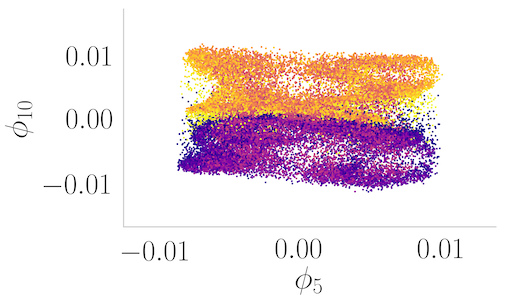}
\label{fig:d14-emb-rank-1-axis-3-10} &
\includegraphics[width=0.24\textwidth]{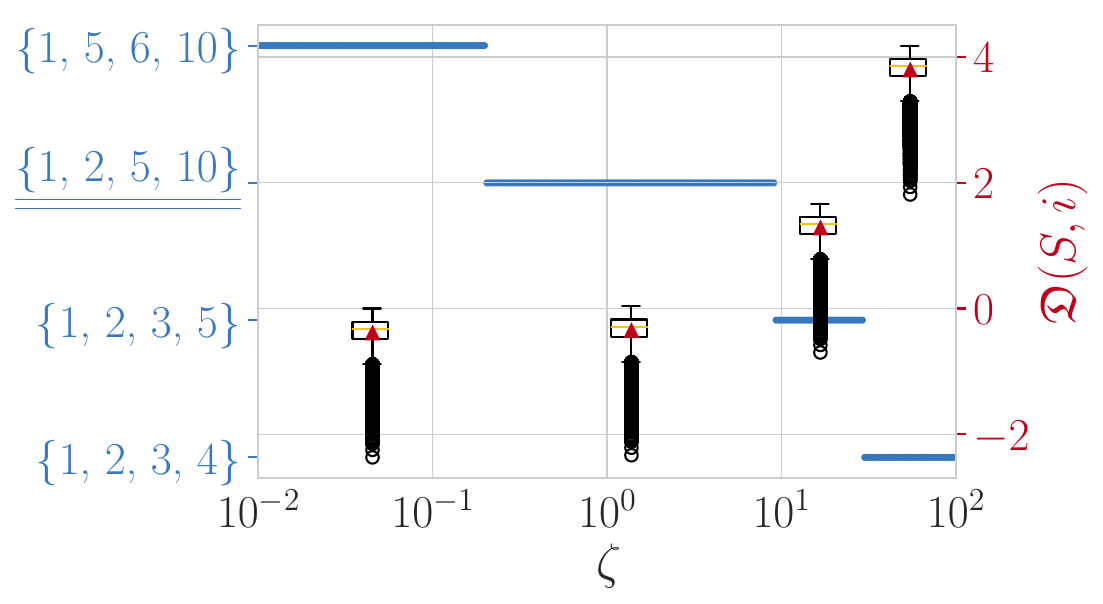}
\label{fig:d14-regu-path}

\end{tabular}

\caption{Experimental result for synthetic datasets. Rows correspond to different synthetic datasets (please refer to Table \ref{tab:dataset-notations}). Optimal subset $S_*$ is selected by \coordsearch. }

\label{fig:synth-data-all}
\end{figure*}

We demonstrate the proposed algorithm on three synthetic datasets, one where the minimum embedding dimension  $s$ equals $d$ ($\mathcal D_1$ {\em long strip}), and two ($\mathcal D_7$ {\em high torus} and $\mathcal D_{13}$ {\em three torus}) where $s>d$.
The complete list of synthetic manifolds (transformations of 2 dimensional strips, 3 dimensional
cubes, two and three tori, etc.) investigated can be found in Supplement \ref{sec:addi-exps} and Table \ref{tab:dataset-notations}.
The examples have (i) aspect ratio of at least 4 (ii) points sampled 
{\em non-uniformly} from the underlying manifold $\mathcal M$, and (iii) Gaussian noise added. The sample size of the synthetic datasets is $n = 10, 000$ unless otherwise stated. 
Additionally, we  analyze several real datasets from chemistry and astronomy. All embeddings are computed with the DM algorithm, which outputs $m = 20$ eigenvectors. Hence, we examine 171 sets for $s=3$ and $969$ sets for $s=4$. No more than 2 to 5 of these sets appear on the regularization path. Detailed experimental results are in Table \ref{tab:result-comp-synth-data}. 
In this section, we show the original dataset $\vect{X}$, the embedding $\phi_{S_*}$, with 
$S_*$ selected by \coordsearch~and $\zeta_*$ from \regusearch, and the maximizer sets on the regularization path with box plots of $\mathfrak D(S, i)$ as discussed in Section \ref{sec:algorithm}. The $\alpha$ threshold for \regusearch~is set to $75\%$. 
All the experiments are replicated for more than 5 times, and the outputs are similar 
because of the large sample size $n$. 

{\bf Synthetic manifolds}
The results of synthetic manifolds are in Figure \ref{fig:synth-data-all}. 
\underline{(i) Manifold with $s = d$.} 
The first synthetic dataset we considered, $\mathcal D_1$, is a two dimensional strip with 
aspect ratio 
$W/H = 2\pi$. Left panel of the top row shows the scatter plot of such dataset. From the
theoretical analysis in Section \ref{sec:problem-formulation}, the coordinate set that 
corresponds to 
slowest varying unique eigendirection is $S = \{1, \lceil W/H\rceil\} = \{1, 7\}$. 
Middle panel,
with $S_* = \{1, 7\}$ selected by \coordsearch~with $\zeta$ chosen by \regusearch, confirms this. 
The right panel shows the box plot of $\{\mathfrak D(S, i)\}_{i=1}^n$. According to the proposed
procedure, we eliminate $S_0 = \{1, 2\}$ since $\mathfrak D(S_0, i)\geq 0$ for almost all the points.
\underline{(ii) Manifold with $s > d$.}
The second data $\mathcal D_7$ is displayed in the left panel of the second row.
Due to the mechanism we used to generate the data,
the resultant torus is non-uniformly distributed along the z axis. 
Middle panel is the embedding of the optimal coordinate set $S_* = \{1, 4, 5\}$
selected by \coordsearch. 
Note that the middle region (in red) is indeed a two dimensional narrow tube when zoomed in. 
The right panel indicates that both $\{1, 2, 3\}$ and $\{1, 2, 4\}$ (median is around zero)
should be removed. The optimal regularization parameter is $\zeta_*\approx 7$. 
The result of the third dataset $\mathcal D_{13}$, {\em three torus}, is in the third row
of the figure. We displayed only projections of the penultimate and the last coordinate of 
original data $\vect{X}$ and embedding $\phi_{S_*}$ (which is $\{5, 10\}$) 
colored by $\alpha_1$ of \eqref{eq:three-torus-parameterization} 
in the left and middle panel to conserve space. 
A full combinations of coordinates can be found in 
Figure \ref{fig:synth-data-3-torus}. 
The right panel implies one should eliminate the set $\{1, 2, 3, 4\}$
and $\{1, 2, 3, 5\}$ since both of them have more than $75\%$ of the points 
such that $\mathfrak D(S, i) \geq 0$. The first remaining subset is $\{1, 2, 5, 10\}$, 
which yields an optimal 
regularization parameter $\zeta_* \approx 5$. 

\begin{figure*}[t]
\subfloat[][Embedding ${\phi}_{[3]}$]
{\includegraphics[width=0.24\textwidth, trim={2.5cm, 0cm, 0cm, 2cm}, clip]{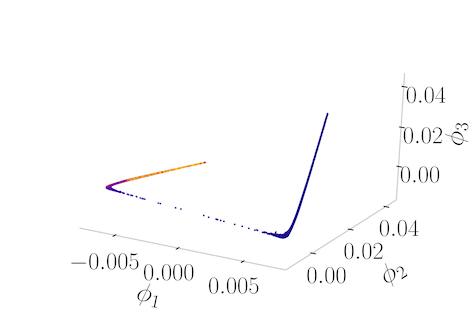}
\label{fig:chloro-first-three}}
\subfloat[][$\loss(\{1, i, j\})$]
{\includegraphics[width=0.24\textwidth]{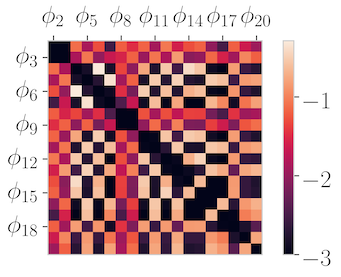}
\label{fig:chloro-loss}}
\subfloat[][$\phi_{S_1} = \phi_{\{1, 4, 6\}}$]
{\includegraphics[width=0.24\textwidth, trim={2.5cm, 0cm, 0cm, 2cm}, clip]{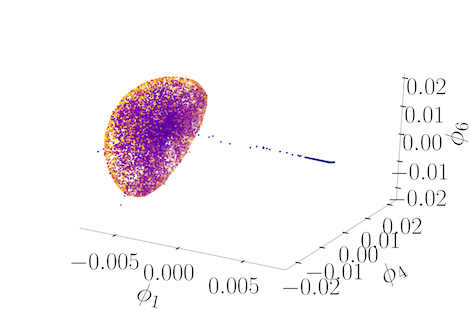}
\label{fig:chloro-1-4-6}}
\subfloat[][$\phi_{S_2} = \phi_{\{1, 5, 7\}}$]
{\includegraphics[width=0.24\textwidth, trim={2.5cm, 0cm, 0cm, 2cm}, clip]{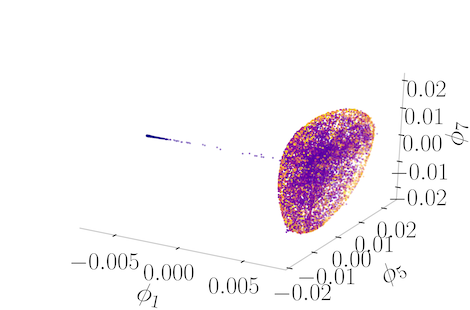}
\label{fig:chloro-1-5-7}}\hfill

\subfloat[][$\loss(\{1, 2\}) = -1.24$]
{\includegraphics[width=0.24\textwidth]{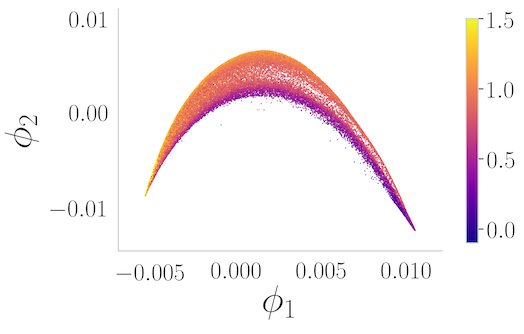}
\label{fig:sdss-coord-1-2}}\hfill
\subfloat[][$\loss(\{1, 3\}) = -0.39$]
{\includegraphics[width=0.24\textwidth]{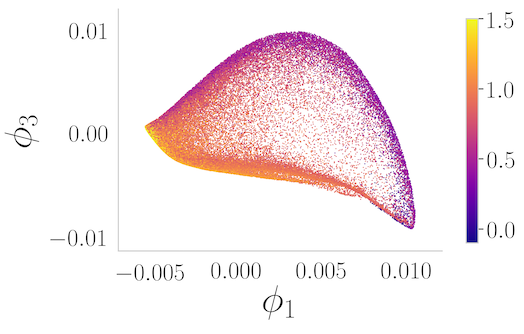}
\label{fig:sdss-coord-1-3}}\hfill
\subfloat[][Subset $\phi_{\{1, 2, 5\}}$ by LLR]
{\includegraphics[width=0.24\textwidth, trim={2.5cm, 0cm, 0cm, 2cm}, clip]{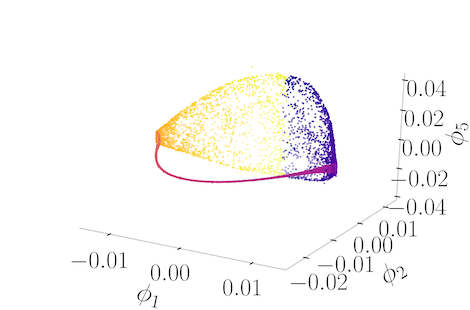}
\label{fig:d7-llr-rank-1}}\hfill
\subfloat[][$r_k$ vs. $\phi_k$]
{\includegraphics[width=0.24\textwidth]{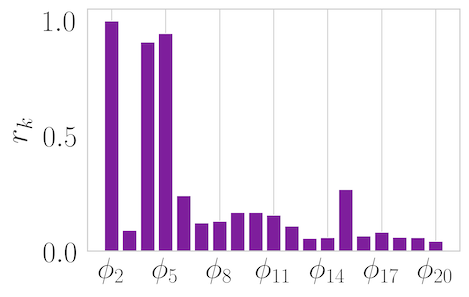}
\label{fig:d7-loo-error}}\hfill

\caption{First row: Choloromethane dataset; second row: SDSS dataset in (e), (f) and (g), (h) show the example when LLR failed.
(c) and (d) are embeddings with top two ranked subsets $S_1$ and $S_2$, colored by the distances between \ce{C} and two different \ce{Cl-}, respectively.
(e) and (f) are embeddings of $\vect{\phi}_{\{1, 2\}}$ (suboptimal set) and $\vect{\phi}_{\{1, 3\}}$ (maximizer of $\loss$), respectively (values shown in caption). 
}
\label{fig:rest-images-table}
\end{figure*}

\paragraph{Molecular dynamics dataset \cite{flemingTPfae:16}}
In SN2 reaction molecular dynamics of chloromethane \cite{flemingTPfae:16} dataset, two chloride atoms substitute with each other in different configurations/points $\xx_i$ as described in the following chemical equation \ce{CH3Cl + Cl^- <-> CH3Cl + Cl^-}. 
The dataset exhibits some kind of clustering structure with a sparse connection between two clusters which represents the time when the substitution happened. 
The dataset has size $n\approx 30,000$ and ambient dimension $D = 40$, with
the intrinsic dimension estimate be $\hat{d} = 2$ 
The embedding with coordinate set $S = [3]$ is shown in Figure \ref{fig:chloro-first-three}. 
The first three eigenvectors parameterize the same directions, which yields a one dimensional manifold in the figure.  Top view ($S = [2]$) of the figure is a u-shaped structure similar to the
yellow curve in Figure \ref{fig:cont-long-strip-27}. 
The heat map of $\loss(\{1, i, j\})$ for different combinations of coordinates in Figure \ref{fig:chloro-loss} confirms that $\loss$ for $S = [3]$ is low and that $\phi_1$, $\phi_2$ and $\phi_3$ give a low rank mapping.
The heat map also shows high $\loss$ values  for $S_1 = \{1, 4, 6\}$ or $S_2 = \{1, 5, 7\}$, which correspond to the top two ranked subsets. 
The embeddings with $S_1, S_2$ are in Figures \ref{fig:chloro-1-4-6} and \ref{fig:chloro-1-5-7}, respectively. In this case, we obtain two optimal $S$ sets due to the data symmetry.

\paragraph{Galaxy spectra from the Sloan Digital Sky survey (SDSS)}\footnote{The Sloan Digital Sky Survey data can be downloaded from \url{https://www.sdss.org}} \cite{abazajian2009seventh}, preprocessed as in \cite{JMLR:v17:16-109}. 
 We display a sample of $n = 50,000$ points from the first $0.3$ million points which correspond to closer galaxies. Figures \ref{fig:sdss-coord-1-2}  and \ref{fig:sdss-coord-1-3} show that the first two coordinates are almost dependent; the embedding with $S_* = \{1, 3\}$ is selected by \coordsearch~with $d = 2$.
Both plots are colored by the blue spectrum magnitude, which is correlated to the number of young stars in the galaxy, showing that this galaxy property varies smoothly  and non-linearly with $\phi_1,\phi_3$, but is not smooth w.r.t. $\phi_1,\phi_2$.

\paragraph{Comparison with \cite{dsilva2018parsimonious}}
The \llrcoordsearch~method outputs similar candidate coordinates as our proposed algorithm most of the time (see Table \ref{tab:result-comp-synth-data}). However, the results differ for {\em high torus} as in Figure \ref{fig:rest-images-table}. Figure \ref{fig:d7-loo-error} is the leave one out (LOO) error $r_k$ versus coordinates. The coordinates chosen by \llrcoordsearch~was $S = \{1, 2, 5\}$, as in Figure \ref{fig:d7-llr-rank-1}. The embedding is clearly shown to be suboptimal, for it failed to capture the cavity within the torus. 
This is because the algorithm searches in a sequential fashion; the noise eigenvector  $\phi_2$ in this example appears before the signal eigenvectors e.g., $\phi_4$ and $\phi_5$.

\paragraph{Additional experiments with real data} are shown in Table \ref{tab:real-dataset-runtime}. Not surprisingly, for most real data sets we examined, the independent coordinates are not the first $s$. They also show that the algorithm scales well and is robust to the noise present in real data.

\begin{table}[!htb]
\caption{Results for other real datasets. Columns from left to right are sample size $n$, ambient dimension of data $D$, average degree of neighbor graph $\mathrm{deg}_\mathrm{avg}$, $(s, d)$ and runtime for IES, and the chosen set $S^*$, respectively. Last three datasets are from \cite{chmiela2017machine}.}
\label{tab:real-dataset-runtime}
\centering
\begin{tabular}{lrrrrrr}
\toprule
{} &      $n$ &   $D$ & $\mathrm{deg}_\mathrm{avg}$ & $(s, d)$ & $t$ (sec) & $S^*$ \\
\midrule
SDSS (full)     &  298,511 &  3750 &                      144.91 &   (2, 2) &    106.05 & (1, 3)\\
Aspirin         &  211,762 &   244 &                      101.03 &   (4, 3) &     85.11 & (1, 2, 3, 7)\\
Ethanol         &  555,092 &   102 &                      107.27 &   (3, 2) &    233.16 & (1, 2, 4)\\
Malondialdehyde &  993,237 &    96 &                      106.51 &   (3, 2) &    459.53 & (1, 2, 3)\\
\bottomrule
\end{tabular}
\end{table}

The asymptotic runtime of \llrcoordsearch~has quadratic dependency on $n$, while for our algorithm is linear in $n$. 
Details of runtime analysis are Supplement \ref{sec:complexity-analysis}. \llrcoordsearch~was too slow to be tested on the four larger datasets (see also Figure \ref{fig:runtime-alg}).

%% file: arxiv-conclusion.tex
\section{Conclusion}
Algorithms that use eigenvectors, such as DM, are among the most promising and well studied in ML. It is known since \cite{goldberg08} that when the aspect ratio of a low dimensional manifold exceeds a threshold, the choice of eigenvectors becomes non-trivial, and that this threshold can be as low as 2. Our experimental results confirm the need to augment ML algorithms with IES methods in order to successfully apply ML to real world problems. Surprisingly, the IES problem has received little attention in the ML literature, to the extent that the difficulty and complexity of the problem have not been recognized. Our paper advances the state of the art by (i) introducing  for the first time a differential geometric definition of the problem, (ii) highlighting geometric factors such as injectivity radius that, in addition to aspect ratio, influence the number of eigenfunctions needed for a smooth embedding, (iii) constructing selection criteria based on {\em intrinsic manifold quantities}, (iv) which have analyzable asymptotic limits, (v) can be computed efficiently, and (vi) are also robust to the noise present in real scientific data. The library of hard synthetic examples we constructed will be made available along with the python software implementation of our algorithms.

%% file: arxiv-eigencoords.bbl
\newcommand{\etalchar}[1]{$^{#1}$}

%% file: arxiv-supplementary-mat.tex
\clearpage
\newpage
\appendix

\renewcommand{\thefigure}{\suppnumberprefix\arabic{figure}}
\setcounter{figure}{0}

\renewcommand{\thetable}{\suppnumberprefix\arabic{table}}
\setcounter{table}{0}

\renewcommand{\theequation}{\suppnumberprefix\arabic{equation}}
\setcounter{equation}{0}

\renewcommand{\thealgocf}{\suppnumberprefix\arabic{algocf}}
\setcounter{algocf}{0}

\renewcommand{\thelemma}{\suppnumberprefix\arabic{lemma}}
\setcounter{lemma}{0}

\twocolumn[\section*{\centering Supplement to \\ \titlename }\vspace{2em}]
\section{Notational table}
\begin{table}[!htb]
\begin{adjustwidth}{-50pt}{0pt}
\caption{Notational table}
\label{tab:notation-summary}
\centering
\begin{tabular}{cl}
\toprule
\multicolumn{2}{c}{Matrix operation} \\
\midrule
$\vect{M}$ & Matrix \\
$\vect{m}_i$ & Vector represents the $i$-th row of $\vect{M}$  \\ 
$\vect{m}_{:,j}^T$ & Vector represents the $j$-th column of $\vect{M}$  \\
$m_{ij}$ & Scalar represents $ij$-th element of $\vect{M}$  \\
$[\vect{M}]_{ij}$ & Scalar, alternative notation for $m_{ij}$  \\
$\vect{M}[\alpha, \beta]$ & Submatrix of $\vect{M}$ of index sets $\alpha, \beta$ \\
$\vect{v}$ & Column vector   \\
$v_{i}$ & Scalar represents $i$-th element of vector $\vect{v}$  \\
$[\vect{v}]_{i}$ & Scalar, alternative notation for $v_{i}$   \\
\midrule
\multicolumn{2}{c}{Scalars} \\
\midrule
$n$ & Number of samples  \\
$D$ & Ambient dimension \\
$m$ & Dimension of diffusion embedding \\
$s$ & (Minimum) embedding dimension \\
$d$ & Intrinsic dimension \\
\midrule
\multicolumn{2}{c}{Vectors \& Matrices} \\
\midrule
$\vect{X}$ & Data matrix \\
$\vect{x}_i$ & Point $i$ in ambient space \\
$\vect{Y}$ & Diffusion coordinates \\
$\vect{y}_i$ & Point $i$ in diffusion coordinates \\
$\phi_i$ & The $i$-th diffusion coordinate of all points \\
$\vect{K}$ & Kernel (similarity) matrix \\
$\vect{L}$ & Graph Laplacian \\
$\vect{H}(i)$ & Dual metric at point $i$ \\
$\vect{I}_k$ & Identity matrix in $k$ dimension space \\
$\vect{1}_n$ & All one vector $\in\mathbb R^n$ \\
$\vect{1}_S$ & $[\vect{1}_S]_i = 1$ if $i\in S$ 0 otherwise\\
\midrule
\multicolumn{2}{c}{Miscellaneous} \\
\midrule
$G(V, E)$ & Graph with vertex set $V$ and edge set $E$ \\
$\M$ & Data manifold \\
$\phi(\cdot)$ & Embedding mapping \\
$\loss(S;\zeta)$ & Utilities \\
$\risk$ & Unpenalized utilities \\
$[s]$ & Set $\{1, \cdots, s\}$ \\
$D(\cdot\|\cdot)$ & KL divergence \\
$\jacobian$ & Jacobian \\
$\mathfrak D(S, i)$ & Leave-one-out regret of point $i$ \\
\bottomrule
\end{tabular}
\label{tb:notation-summaries}
\end{adjustwidth}
\end{table}

\newpage
\section{Pseudocodes}
\label{sec:pseudocodes}
\begin{algorithm}[!htb]
    \SetKwInOut{Input}{Input}
	\SetKwInOut{Output}{Return}
	\SetKwComment{Comment}{~~$\bullet$\ }{}
\Input{Data matrix $\vect{X}\in \rrr^{n\times D}$, bandwidth $\varepsilon$, embedding dimension
$m$}
Compute similarity matrix $\vect{K}$ with
$K_{ij}=\begin{cases}\exp\left[-\frac{||\vect{x}_i-\vect{x}_j||^2}{\varepsilon^2}\right] & \text{ if } \|x-y\|\leq 3\varepsilon \\ 0 & \text{ otherwise }\end{cases}$ \\
$\vect{L}\gets\text{\lapalg}(\vect{K})\in \rrr^{n\times n}$ (Algorithm
\ref{alg:graph-laplacian-alg})\\
Compute eigenvectors of $\vect{L}$ for {smallest $m+1$ eigenvalues}
$[\vect{\phi}_0\,\vect{\phi}_1\,\ldots \vect{\phi}_m]\in\rrr^{n\times (m+1)}$ \\
\Output{$\vect{\Phi}=[\vect{\phi}_1\,\ldots \vect{\phi}_m]\in \rrr^{n\times m}$
The \emph{embedding coordinates} of $\vect{x}_i$ are $(\Phi_{i1},\ldots, \Phi_{im}) \in \rrr^m$
}
    \caption{\diffmapalg}
    \label{alg:diffusion-map}
\end{algorithm}

\begin{algorithm}[!htb]
    \SetKwInOut{Input}{Input}
	\SetKwInOut{Output}{Return}
	\SetKwComment{Comment}{$\triangleright $\ }{}
    \Input{Symmetric similarity matrix $\vect{K}$}
    Calculate the {\em degree} of node $i$, $[\vect{w}]_i = \sum_{j=1}^n K_{ij}$ 
    \Comment{Set $\vect{W} = \diag(\vect{w})$}
    $\tilde{\vect{L}} = \vect{W}^{-1}\vect{K}\vect{W}^{-1}$ \\
    $[\tilde{\vect{w}}]_i\gets \sum_{j=1}^n \tilde{L}_{ij}$ 
    \Comment{Set $\tilde{\vect{W}} = \diag(\tilde{\vect{w}})$}
    $\vect{L} = \vect{I}_n - \tilde{\vect{W}}^{-1}\tilde{\vect{L}}$ \\
    \Output{Renormalized graph Laplacian $\vect{L}$}
    \caption{\lapalg}
    \label{alg:graph-laplacian-alg}
\end{algorithm}

\begin{algorithm}[h]
    \SetKwInOut{Input}{Input}
	\SetKwInOut{Output}{Return}
	\SetKwComment{Comment}{$\triangleright $\ }{}
    \Input{Embedding $\vect{Y} = [\phi_1,\cdots,\phi_m] \in \mathbb R^{n\times m}$}
    Set the leave-one-out validation error $\vect{r} = [1, \cdots, 1] \in \mathbb R^m$ \\
    \For{$s = 2\to m$}{
    	Bandwidth of LLR: $h \gets \inv{3}\cdot \textsc{Median}(\textsc{PairwiseDist}(\phi_{[s-1]}))$ \\
    	$\hat{\phi}_s \gets $ {\sc LocalLinearRegression}$(\phi_s, \phi_{[s-1]}, h)$ \\
    	$r_s = \sqrt{\frac{\|\hat{\phi}_s - \phi_s\|^2}{\|\phi_s\|^2}}$
    }
    $S_*\gets \textsc{ArgSort}(\vect{r})$\\
    \Comment{Sort in descending order.}
    \Output{Sorted independent coordinates $S_*$}
    \caption{\llrcoordsearch}
    \label{alg:llr-coord-search}
\end{algorithm}

\newpage
\onecolumn

\section{Extra theorems}
\label{sec:proofs-extra-analysis}
\subsection{Submodularity of the objective functions}
\label{sec:submodularity-loss-func}
\begin{theorem}
\label{thm:sf-loss1-func}
	For a rank $d$ tangent space matrix $\vect{U}\in \mathbb R^{m\times d}$, 
	if any submatrix $\vect{U}_S$, with index set $S\subseteq [m]$ and $|S| = s \geq d$,
	is rank $d$, we have $\risk_1$ be a submodular set function. 
\end{theorem}
\begin{proof}
	W.L.O.G, set $n = 1$, with slightly abuse of 
	notation, let $\vect{U} = \vect{U}_{T\cup\{i\}} \in \mathbb R^{(|T|+1)\times d}$. The matrix can be 
	written in the following form
	\begin{equation}
		\vect{U} = 
		\begin{bmatrix}
			\vect{T} \\
			\vect{a}
		\end{bmatrix} = 
		\begin{bmatrix}
			\vect{S} \\
			\vect{V} \\
			\vect{a}
		\end{bmatrix} \in \mathbb R^{(|T|+1)\times d}
		\nonumber
	\end{equation}
	
	With $\vect{U}_S = \vect{S}$, $\vect{U}_T = \vect{T}$ and $U_{\{i\}} = \vect{a}$ for set $S \subseteq T \subseteq [m]$ 
	and $i\in [m]\backslash T$. 
	Here $\vect{a}\in \mathbb R^{1\times d}$. By the definition of $\risk_1$ in \eqref{eq:loss-func}, one has (ignoring the constants)
	\begin{equation}
	\begin{gathered}
		\risk_1(S) = \log\det(\vect{S}^\top\vect{S}) \\	
		\risk_1(T) = \log\det\left(\vect{T}^\top\vect{T}\right) \\
		\risk_1(S\cap \{i\}) = \log\det\left(\begin{bmatrix}\vect{S}\\ \vect{a}\end{bmatrix}^\top\begin{bmatrix}\vect{S}\\ \vect{a}\end{bmatrix}\right) \\
		\risk_1(T\cap \{i\}) = \log\det(\vect{U}^\top\vect{U})
	\end{gathered}
	\nonumber
	\end{equation} 
	Denote $\partial_i f(S) = f(S\cup\{i\}) - f(S)$ for some function $f$, we have 
	\begin{equation}
	\begin{gathered}
		\partial_i \risk_1(S) = \log\det(\vect{S}^\top\vect{S} + \vect{a}^\top\vect{a}) -\log\det (\vect{S}^\top \vect{S}) \\  
		\partial_i \risk_1(T) = \log\det(\vect{T}^\top\vect{T} + \vect{a}^\top\vect{a}) -\log\det (\vect{T}^\top \vect{T})
	\end{gathered}	
	\nonumber
	\end{equation}
	The full rank of any submatrices guarantees the positive definiteness of 
	$\vect{S}^\top\vect{S}, \vect{T}^\top\vect{T}$, by matrix determinant lemma
	\cite{harville1998matrix}, we have
	\begin{equation}
		\det(\vect{S}^\top\vect{S} + \vect{a}^\top\vect{a}) = \det(\vect{S}^\top\vect{S})\left(1+\vect{a}(\vect{S}^\top\vect{S})^{-1}\vect{a}^\top\right)
		\nonumber
	\end{equation}
	Therefore
	\begin{equation}
		\partial_i \risk_1(S) = 1+\vect{a}(\vect{S}^\top\vect{S})^{-1}\vect{a}^\top
		\nonumber
	\end{equation}
	Similar equation holds for set $T$. Therefore,
	\begin{equation}
		\partial_i \risk_1(S) - \partial_i \risk_1(T) = \log \frac{1+\vect{a}(\vect{S}^\top\vect{S})^{-1}\vect{a}^\top}{1+\vect{a}(\vect{T}^\top\vect{T})^{-1}\vect{a}^\top}
		\nonumber
	\end{equation}
	
	Because $\vect{T}^\top\vect{T}\succeq\vect{S}^\top\vect{S}$, we have 
	$(\vect{S}^\top\vect{S})^{-1}\succeq (\vect{T}^\top\vect{T})^{-1}$ \cite{horn1990matrix},
	which implies $\partial_i \risk_1(S) - \partial_i \risk_1(T) \geq 0$ for all 
	$S\subseteq T\subseteq [m]$ and $i\in [m]\backslash T$. This completes the proof.
\end{proof}
\begin{theorem}
\label{thm:sf-loss2-func}
	$\risk_2$ is a submodular set function.
\end{theorem}
\begin{proof}
	W.L.O.G, set $n, d = 1$. With slightly abuse of notation, let $\vect{u} \gets \vect{u}_1(i)$ and $\vect{u}_S \gets \vect{u}^S_1(i)$. For any set $S \subseteq T \subseteq [m]$ and $i\in [m]\backslash T$, we have 
	\begin{equation}
	\begin{gathered}
	\begin{split}
		\partial_i \risk_2(S) & = \risk_2\left(S \cap \{i\}\right) - \risk_2(S) 
		= \log \frac{\sum_{k\in S} u_k^2 + u_i^2}{\sum_{k\in S}u_k^2}
		= \log \frac{\Sigma_S + u_i^2}{\Sigma_S}
	\end{split} \\
	\begin{split}
		\partial_i \risk_2(T) & = \risk_2\left(T \cap \{i\}\right) - \risk_2(T) = \log \frac{\sum_{k\in T} u_k^2 + u_i^2}{\sum_{k\in T}u_k^2}  = \log \frac{\Sigma_S + \Sigma_{T\backslash S} + u_i^2}{\Sigma_S + \Sigma_{T\backslash S}}
	\end{split}
	\end{gathered}
	\nonumber
	\end{equation}
	
	Where $\Sigma_S = \sum_{k\in S} u_k^2$. By definition, 
	we have $\Sigma_{S}, \Sigma_{T\backslash S}, u_i^2 \geq 0$. Therefore, 
	\begin{equation}
	\begin{split}
		& \partial_i \risk_2(S) - \partial_i \risk_2(T) = \log \frac{\left(\Sigma_S + u_i^2\right)\cdot\left(\Sigma_S + \Sigma_{T\backslash S}\right)}{\Sigma_S\cdot \left(\Sigma_S + \Sigma_{T\backslash S} + u_i^2\right)} \\
		& = \log \underbrace{\left[\frac{\Sigma_S^2 + \Sigma_S\left(\Sigma_{T\backslash S} + u^2_i\right) + u_i^2\Sigma_{T\backslash S}}{\Sigma_S^2 + \Sigma_S\left(\Sigma_{T\backslash S} + u^2_i\right)}\right]}_{\geq 1} \geq 0
	\end{split}
	\nonumber
	\end{equation}
	Which completes the proof. 
\end{proof}

\section{Greedy search}
\label{sec:greedy-algorithm}
\begin{algorithm}[H]
    \setstretch{1.15}
    \SetKwInOut{Input}{Input}
    \SetKwInOut{Output}{Return}
    \SetKwComment{Comment}{$\triangleright $\ }{}
    \Input{Orthogonal basis $\{\vect{U}(i)\}_{i=1}^n$, eigenvalues $\vect{\lambda}$, intrinsic dimension $d$, regularization parameter $\zeta$}
    Solve $S_* \gets \argmax_{S\subseteq[m]; |S|=d; 1\in S} \loss(S;\zeta)$. \\
    \For{$s = d+1\to m$}{
        $k_* = \argmax_{k\in [m] \backslash S_*} \loss(S_*\cup \{k\};\zeta)$ \\
        $S_* \gets S_*\cup \{k_*\}$ 
        \Comment{Record order}
    }
    \Output{Independent coordinates $S_*$}
    \caption{\greedycoordsearch}
    \label{alg:greedy-indep-coord-search}
\end{algorithm}

Inspired by the greedy version of submodular maximization \cite{nemhauser1978analysis}, a greedy heuristic has been proposed, as in Algorithm
\ref{alg:greedy-indep-coord-search}.
The algorithm starts from an observation that the optimal value of the $S' = \argmax_{S; d\leq |S| < s} \loss(S;\zeta)$ will often time be a subset of the optimal $S_*$ of
\eqref{eq:subset-selection-problem}.
Since the appropriate cardinality of the set $S$ is unknown, we can simply scan from $|S| = d$ to $m$. 
The order of the returned elements indicates the significance of the corresponding coordinate. 

\section{Computational complexity analysis}
\label{sec:complexity-analysis}

\subsection{The proposed algorithms}
For computation complexity analysis, we assume the embedding has already been obtained. 
Therefore, the computational complexity for building neighbor graph and solving the
eigen-problem of graph Laplacian can be omitted. This is also the case for \llrcoordsearch.

\paragraph{Co-metrics and orthogonal basis}
According to \cite{2013arXiv1305.7255P}, time complexity for computing $\vect{H}(i)\in\mathbb 
R^{m\times m}~\forall~i\in[n]$  is $\mathcal O(nm^2\delta)$, with $\delta$ be the average degree
of the neighbor graph $G(V, E)$. In manifold learning, the graph will be sparse 
therefore $\delta \ll n$. Time complexity for obtaining principal space $\vect{U}(i)$
of point $i$ via SVD will be $\mathcal O(m^3)$. Total time complexity will be 
$\mathcal O(nm^2\delta + nm^3)$.

\paragraph{Exact search} 
Evaluating the objective function $\loss$ for each point $i$ takes $\mathcal O(sd^2)$ 
in computing $\vect{U}_S(i)^\top \vect{U}_S(i)$, $\mathcal O(d^3)$ in 
evaluating the determinant of a $d\times d$ matrix. Normalization ($\risk_2$ term)
takes $\mathcal O(ds)$. Exhaustive search over all the subset with cardinality $s$
takes $\mathcal O\left(\binom{m}{s}\right)$. The total computational complexity will 
therefore be 
$\mathcal O(nm^s(d^3+d^2s) + nm^2\delta+nm^3) = \mathcal O(nm^{s+3} + nm^2\delta)$.

\paragraph{Greedy algorithm}
First step of greedy algorithm includes solving $\argmax_{S\subseteq [m]; |S|=d} \loss(S, d)$, which takes $\mathcal O(nm^dd^3) = \mathcal O(nm^{d+3})$. Starting from $s = d+1\to m$,
each step includes exhaustively search over $m-s$ candidates, with the time complexity
of evaluating $\loss$ be $n(d^3+d^2s)$. Putting things together, one has the second part of 
the greedy algorithm be
\begin{equation}
	\sum_{s=d}^m n(m-s)(d^3+d^2s) = \mathcal O(nm^5)
\end{equation}

The total computational complexity will therefore be $\mathcal O(n(m^{d+3} + m^5 + m^2\delta))$.

\subsection{Time complexity of \cite{dsilva2018parsimonious} \& discussion}
\label{sec:comp-llr}
The Algorithm \llrcoordsearch~is summarized in Algorithm \ref{alg:llr-coord-search}.
For searching over fixed coordinate $s$, 
the algorithm first build a kernel for local linear regression 
by constructing a neighbor graph, 
which takes $\mathcal O(n\log(n)s)$\footnote{
This is a simplified lower bound, see \cite{dasgupta2013randomized} for details.}
using approximate nearest neighbor search. 
The $s$ dependency come from the dimension of the feature.
For each point $i$, a ordinary least square (OLS) problem is solved, which results
in $\mathcal O(n^2s^2+ns^3)$ time complexity. 
\begin{wrapfigure}[15]{r}{0.5\textwidth}
    \centering
    \includegraphics[width=0.8\linewidth]{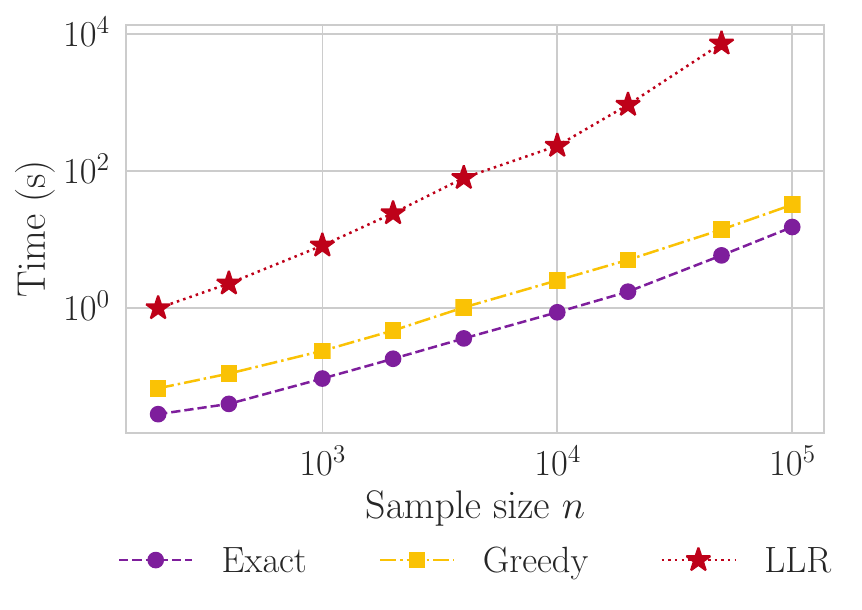}
    \caption{Runtimes of different IES algorithms on two dimensional long strip. Purple, yellow and red curves correspond to \coordsearch, \greedycoordsearch~and \llrcoordsearch~algorithm, respectively.}
    \label{fig:runtime-alg}
\end{wrapfigure}
Searching from $s = 2\to m$ will 
make the total time complexity be
\begin{equation}
	\sum_{s=2}^m n^2s^2 + ns^3 + ns\log n = \mathcal O(n^2m^3 + nm^4)
\end{equation}

For a sparse graph, the overheads of the \coordsearch~and \greedycoordsearch~algorithms
come from the enumeration of the subset $S$. Because of the linear dependency on the
sample size $n$, the algorithm is tractable for small $s$ and $d$. However, \llrcoordsearch~
has a quadratic dependency on sample size $n$, which is more computationally 
intensive for large sample size.
For large $s$ and $d$, one can use the techniques in difference between submodular function 
optimization (e.g. \cite{Iyer:2012:AAM:3020652.3020697}) as $\risk_1$, $\risk_2$ are both 
submodular set function from Theorems \ref{thm:sf-loss1-func} and \ref{thm:sf-loss2-func}. 
An empirical runtime plot for different algorithms 
can be found in Figure \ref{fig:runtime-alg}. The runtime was 
evaluated on two dimensional long strip with $s = d = 2$ and 
was performed on a single desktop computer running Linux with 
32GB RAM and a 8-Core 4.20GHz Intel\textregistered~Core\texttrademark~i7-7700K CPU.

\section{A discussion on UMAP}
\begin{figure*}[t]
\centering
\subfloat[][]
{\includegraphics[width=0.54\linewidth]{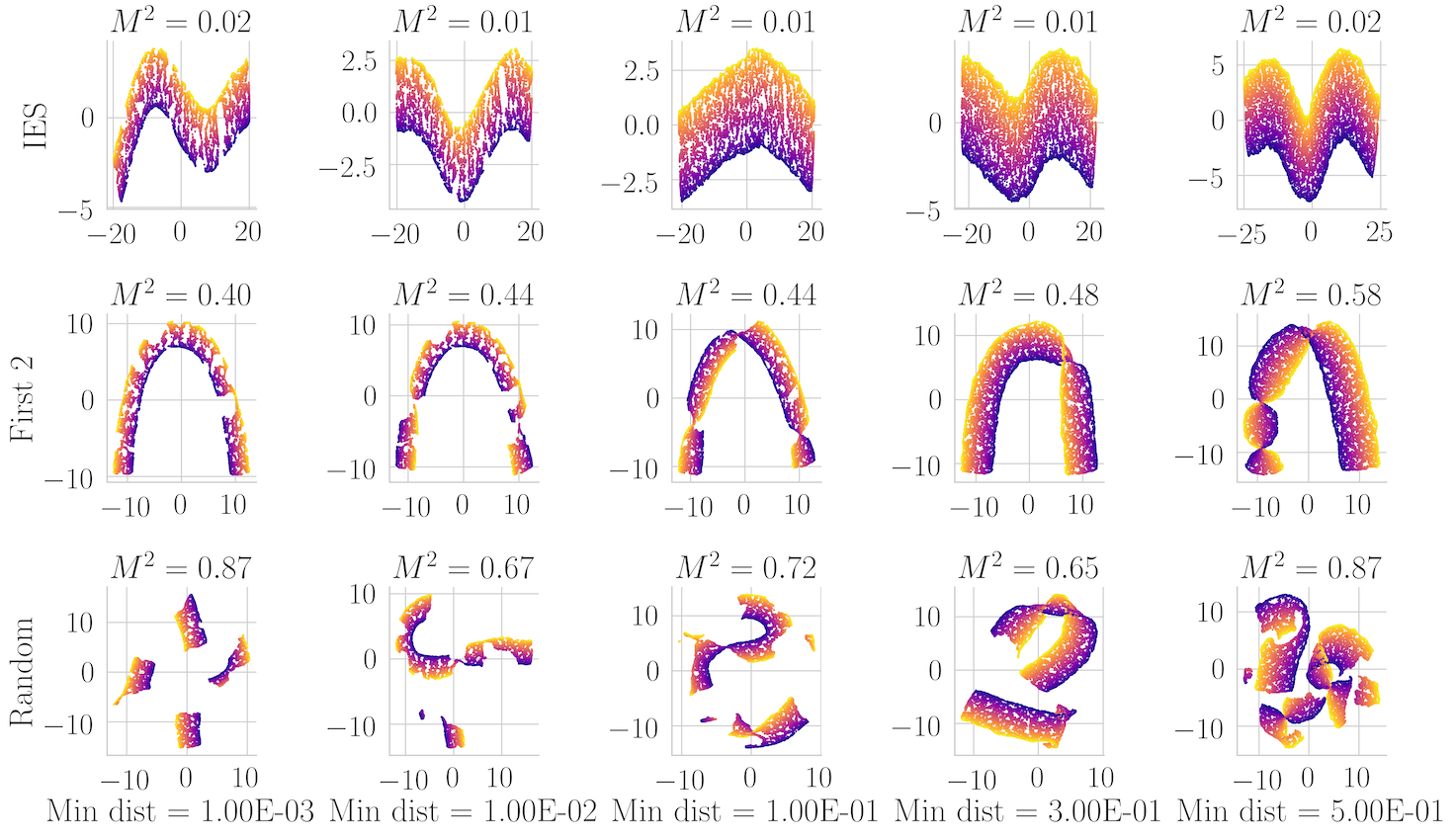}
\label{fig:umap-mindist}}\hfill
\subfloat[][]
{\includegraphics[width=0.43\linewidth]{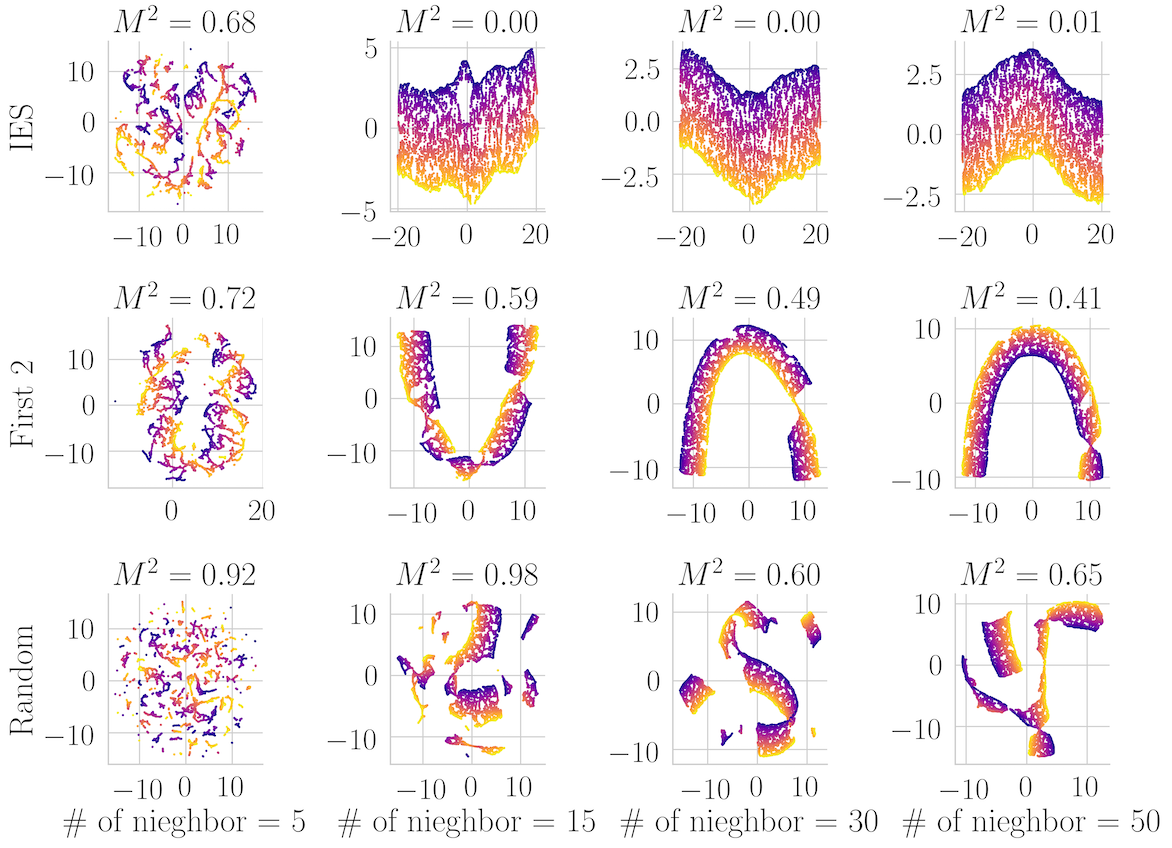}
\label{fig:umap-kneigh}}\hfill

\caption{UMAP embeddings of 2D long stripe with different initializations and choices of hyper-parameters. Rows from top to bottom correspond to UMAP embedding initialized with DM which coordinates chosen by \coordsearch, na\"ive DM and random initialization, respectively. Columns represent different choices of (a) points separation and (b) number of neighbors.}
\label{fig:umap-discussion}	
\end{figure*}

UMAP \cite{mcinnes2018umap} is a commonly used data visualization alternative of t-SNE. 
The authors proposed to use the spectral embedding of the graph Laplacian as an initialization to
the algorithm for faster convergence (compared to random initialization). 
In this section, we showed empirically that, (1) given reasonable computing resources, the 
IES problem also appears in the UMAP embedding and (2) by initializing with spectral 
embedding with carefully selected coordinate set chosen by 
\coordsearch, one can obtain a faster convergence and a globally interpretable embedding. 
Figure \ref{fig:umap-discussion} is the UMAP embedding of 2D long stripe dataset $\mathcal D_1$
with different choices of hyper-parameters (points separation in \ref{fig:umap-mindist} and 
number of neighbors in \ref{fig:umap-kneigh}), 
with total number of epochs be $500$. The first row of both plots are the embedding initialized
with \coordsearch{}, the second row corresponds to those initialized with na\"ive DM. 
The embeddings in the third row are initialized randomly. As shown in the results, 
algorithmic/random artifacts can be easily seen in the embeddings with Na\"ive DM or
random initialization (2nd and 3rd rows).
More precisely, unwanted patterns which reduce the interpretability of the embeddings, 
e.g., the ``knots'' in the second row or disconnected components in the second/third rows, 
are generated. 
The sum of square procrustes error $M^2$
between ground truth dataset and the embedding shown on each subplots also confirm our statement. 
Note that it is possible to unroll the algorithmic artifact with more epochs in the
sampling steps of UMAP. (3 to 5 times more iterations are needed in this example.) 
However, due to the efficiency of performing \coordsearch, it is beneficial 
to initialize with the embedding selected by \coordsearch{}.

\section{Additional experiments \& details of the used datasets}
\label{sec:addi-exps}

In this paper, a total of 13 different synthetic manifolds are considered. Table 
\ref{tab:dataset-notations} summarized the synthetic manifolds constructed and 
its abbreviations (from $\mathcal D_1$ to $\mathcal D_{13}$). 
Embedding results for the synthetic manifolds are in Figures
\ref{fig:synth-data-s-d-equal-addi}, \ref{fig:synth-data-s-geq-d-addi} and 
\ref{fig:synth-data-3-torus}. The ranking of the first few candidate sets $S$ 
from \coordsearch, \greedycoordsearch and \llrcoordsearch~can be found in 
Table \ref{tab:result-comp-synth-data}. The table shows the optimal subsets
return by three different algorithms are often time the same, with exception
for $\mathcal D_7$ {\em high torus} as discussed in Section \ref{sec:experiments}.  

\begin{table}[h]
\caption{Abbreviations for different synthetic manifolds in this paper. The abbreviation with asterisk represents such dataset is discussed in main manuscript.}
\label{tab:dataset-notations}
\centering
\begin{tabular}{cl}
\toprule
\multicolumn{2}{c}{Manifold with $s = d$} \\
\midrule
$\mathcal D_1^*$ & Two dimensional strip (aspect ratio $2\pi$) \\
$\mathcal D_2$ & 2D strip with cavity (aspect ratio $2\pi$) \\
$\mathcal D_3$ & Swiss roll \\
$\mathcal D_4$ & Swiss roll with cavity \\
$\mathcal D_5$ & Gaussian manifold \\
$\mathcal D_6$ & Three dimensional cube \\
\midrule
\multicolumn{2}{c}{Manifold with $s > d$} \\
\midrule
$\mathcal D_7^*$ & High torus \\
$\mathcal D_8$ & Wide torus \\
$\mathcal D_9$ & z-asymmetrized high torus\\
$\mathcal D_{10}$ & x-asymmetrized high torus \\
$\mathcal D_{11}$ & z-asymmetrized wide torus \\
$\mathcal D_{12}$ & x-asymmetrized wide torus \\
$\mathcal D_{13}^*$ & Three-torus \\
\bottomrule
\end{tabular}
\end{table}

\begin{table}
\begin{adjustwidth}{-50pt}{-50pt}
\caption{Results returned from different algorithms on different synthetic datasets.}
\label{tab:result-comp-synth-data}
\centering
\begin{tabular}{llllllll}
\toprule
{} & \multicolumn{5}{c}{Exact search} &             Greedy rank &                   LLR rank \\
{} &              1 &              2 &              3 &              4 & 5 & \multicolumn{2}{c}{} \\
\midrule
$\mathcal D_{1}$  &         [1, 7] &         [1, 8] &         [1, 9] &        [1, 10] &         [1, 12] &   [1, 7, 6, 4, 3, 2, 5] &  [1, 7, 14, 16, 11, 18, 6] \\
$\mathcal D_{2}$  &         [1, 4] &         [1, 8] &         [1, 9] &        [1, 10] &         [1, 12] &   [1, 4, 8, 6, 5, 3, 2] &   [1, 4, 8, 5, 17, 11, 14] \\
$\mathcal D_{3}$  &         [1, 9] &        [1, 10] &        [1, 11] &        [1, 13] &         [1, 18] &   [1, 9, 5, 2, 3, 4, 6] &  [1, 9, 19, 16, 12, 10, 4] \\
$\mathcal D_{4}$  &         [1, 8] &        [1, 10] &        [1, 11] &        [1, 14] &         [1, 15] &  [1, 8, 3, 2, 4, 10, 5] &  [1, 8, 11, 10, 19, 16, 4] \\
$\mathcal D_{5}$  &         [1, 6] &         [1, 8] &        [1, 10] &        [1, 11] &         [1, 13] &  [1, 6, 2, 8, 3, 10, 4] &  [1, 6, 19, 8, 18, 14, 12] \\
$\mathcal D_{6}$  &      [1, 2, 8] &     [1, 2, 11] &      [1, 4, 8] &     [1, 2, 17] &      [1, 2, 13] &   [1, 2, 8, 3, 4, 6, 5] &    [1, 2, 8, 10, 3, 13, 6] \\
$\mathcal D_{7}$  &      [1, 4, 5] &      [1, 4, 8] &      [1, 5, 7] &     [1, 7, 12] &       [1, 7, 8] &   [1, 5, 4, 3, 6, 2, 8] &    [1, 2, 5, 4, 15, 6, 10] \\
$\mathcal D_{8}$  &      [1, 2, 7] &      [1, 4, 7] &      [1, 3, 7] &      [1, 2, 9] &       [1, 5, 7] &  [1, 7, 2, 4, 3, 13, 5] &  [1, 2, 7, 13, 12, 15, 14] \\
$\mathcal D_{9}$  &      [1, 3, 4] &      [1, 3, 7] &      [1, 4, 6] &     [1, 3, 10] &       [1, 7, 9] &   [1, 3, 4, 2, 9, 7, 6] &     [1, 3, 4, 2, 19, 8, 7] \\
$\mathcal D_{10}$ &      [1, 2, 4] &      [1, 3, 4] &      [1, 4, 5] &      [1, 6, 9] &      [1, 6, 14] &   [1, 4, 2, 3, 5, 6, 8] &      [1, 4, 2, 3, 8, 5, 6] \\
$\mathcal D_{11}$ &      [1, 2, 5] &      [1, 4, 8] &      [1, 4, 5] &      [1, 8, 9] &       [1, 2, 8] &   [1, 5, 2, 4, 8, 3, 9] &    [1, 2, 5, 8, 10, 9, 11] \\
$\mathcal D_{12}$ &      [1, 2, 5] &      [1, 4, 5] &      [1, 2, 7] &      [1, 3, 5] &       [1, 2, 8] &   [1, 5, 2, 3, 4, 6, 8] &     [1, 5, 2, 6, 10, 9, 4] \\
$\mathcal D_{13}$ &  [1, 2, 5, 10] &  [1, 3, 5, 10] &  [1, 4, 5, 10] &  [1, 5, 6, 10] &   [1, 2, 8, 10] &  [1, 5, 10, 2, 4, 3, 6] &  [1, 2, 10, 5, 14, 15, 16] \\
\bottomrule
\end{tabular}
\end{adjustwidth}
\end{table}

\subsection{Additional experiments on synthetic manifolds with $s = d$}
Below summarized the details of generating the datasets.
\begin{enumerate}
	\item $\mathcal D_1^*$: points from this dataset are sampled uniformly from
		$\vect{x}_i \sim \textsc{Unif}([-2, 2]\times [-4\pi, 4\pi])$.
	\item $\mathcal D_2$: points are first sampled uniformly from $[-2, 2]\times [-4\pi, 4\pi]$. 
		Points $i$ are removed if $|X_{i1}|<4\pi/3$ and $|X_{i2}|<2/3$. 
	\item $\mathcal D_3$: first sampling points 
		$\vect{X}_\mathrm{true} = [\vect{x}_0, \vect{y}_0]$ 
		uniformly from a two dimensional strip. The data $\vect{X}$ can be obtained by the 
		following non-linear transformation.
		\begin{equation}
			\vect{X} = \left[\frac{\vect{x}_0\scirc \cos \vect{x}_0}{2}, \vect{y}_0, \frac{\vect{x}_0\scirc\sin \vect{x}_0}{2} \right]
		\label{eq:swissroll-trans}
		\end{equation}
		
		With $\scirc$ denotes Hadamard (element-wise) product. 
	\item $\mathcal D_4$: sampling points $\vect{X}_\mathrm{true} = [\vect{x}_0, \vect{y}_0]$ 
		uniformly from 2D strip with cavity then applying the transformation 
		\eqref{eq:swissroll-trans} to get $\vect{X}$. 
	\item $\mathcal D_5$: sampling points $\vect{X}_\mathrm{true}$ uniformly from ellipse $\left\{(x, y)\in\mathbb R^2: \left(\frac{x}{6}\right)^2 + \left(\frac{y}{2}\right)^2 = 1\right\}$. The data is obtained by 
		\begin{equation}
			\vect{X} = \left[\vect{X}_\mathrm{true},  \vect{z}\right]
			\nonumber
		\end{equation}
		With $z_i = \exp\left(-\left(\left(\frac{X_{i1}}{3}\right)^2 + X_{i2}^2\right)/2\right)$
	\item $\mathcal D_6$: points are sampled uniformly from $[-1, 1]\times[-2, 2]\times [-4, 4]$.
\end{enumerate}

The experimental results are in Figure \ref{fig:synth-data-s-d-equal-addi} ($\mathcal D_4$
in Figure \ref{fig:d4-embedding-otherview}).

\begin{figure}[htb]
\subfloat[][]
{\includegraphics[width=0.32\textwidth]{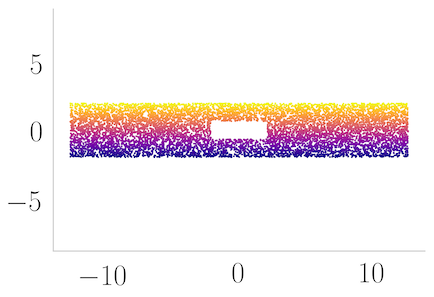}
\label{fig:d2-orig-data}}\hfill
\subfloat[][]
{\includegraphics[width=0.32\textwidth]{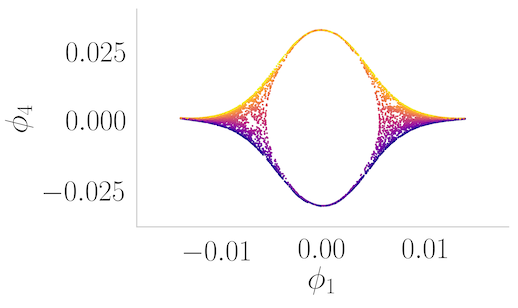}
\label{fig:d2-emb-rank-1}}\hfill
\subfloat[][]
{\includegraphics[width=0.32\textwidth]{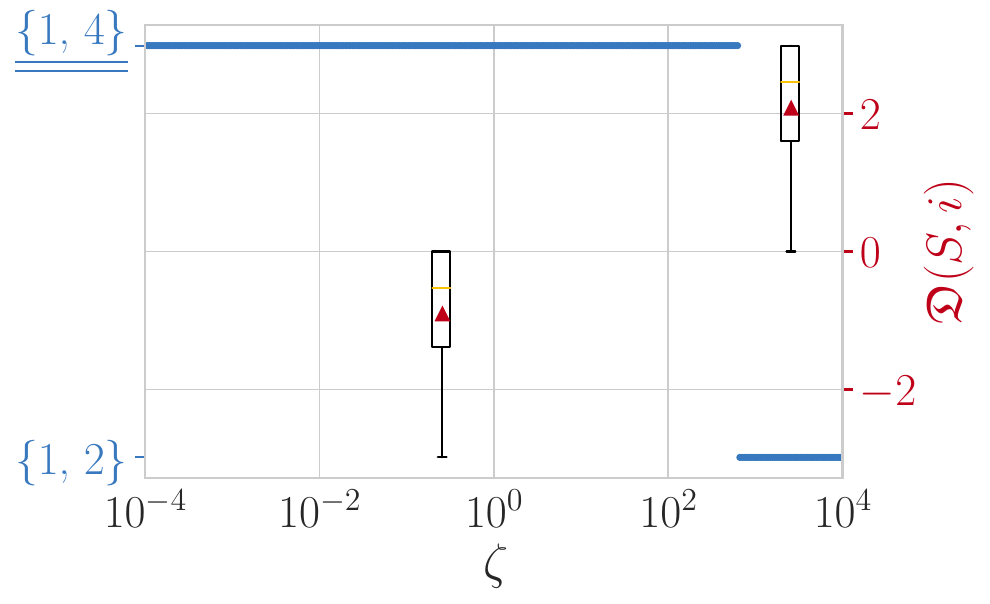}
\label{fig:d2-regu-path}}
\hfill
\subfloat[][]
{\includegraphics[width=0.32\textwidth]{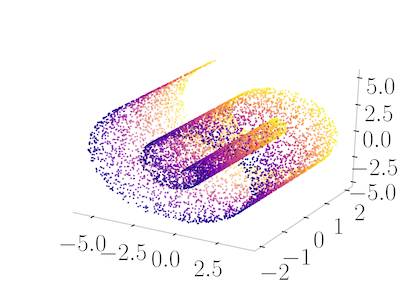}
\label{fig:d3-orig-data}}\hfill
\subfloat[][]
{\includegraphics[width=0.32\textwidth]{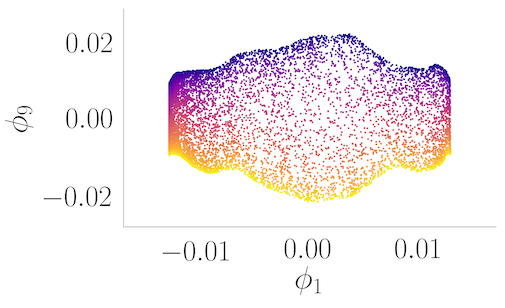}
\label{fig:d3-emb-rank-1}}\hfill
\subfloat[][]
{\includegraphics[width=0.32\textwidth]{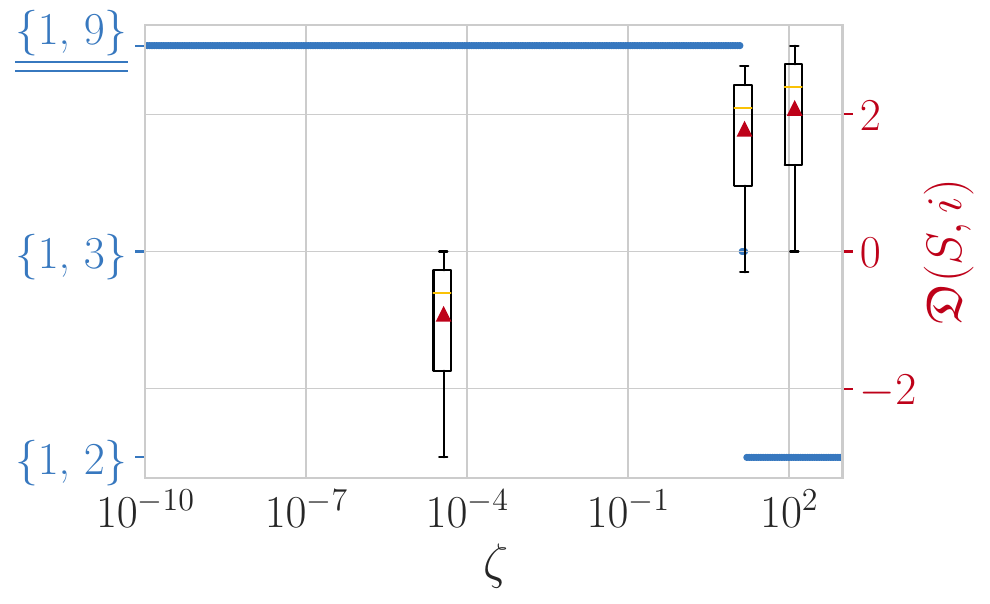}
\label{fig:d3-regu-path}}
\hfill
\subfloat[][]
{\includegraphics[width=0.32\textwidth]{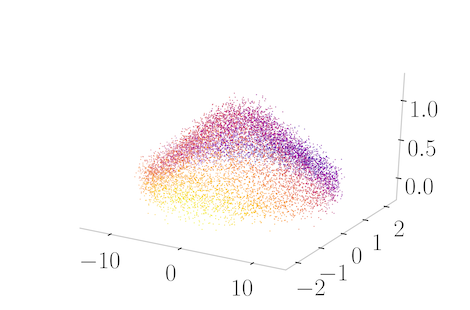}
\label{fig:d5-orig-data}}\hfill
\subfloat[][]
{\includegraphics[width=0.32\textwidth]{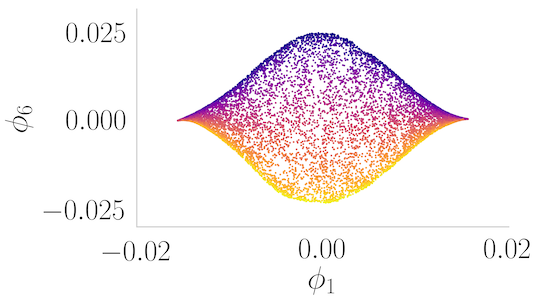}
\label{fig:d5-emb-rank-1}}\hfill
\subfloat[][]
{\includegraphics[width=0.32\textwidth]{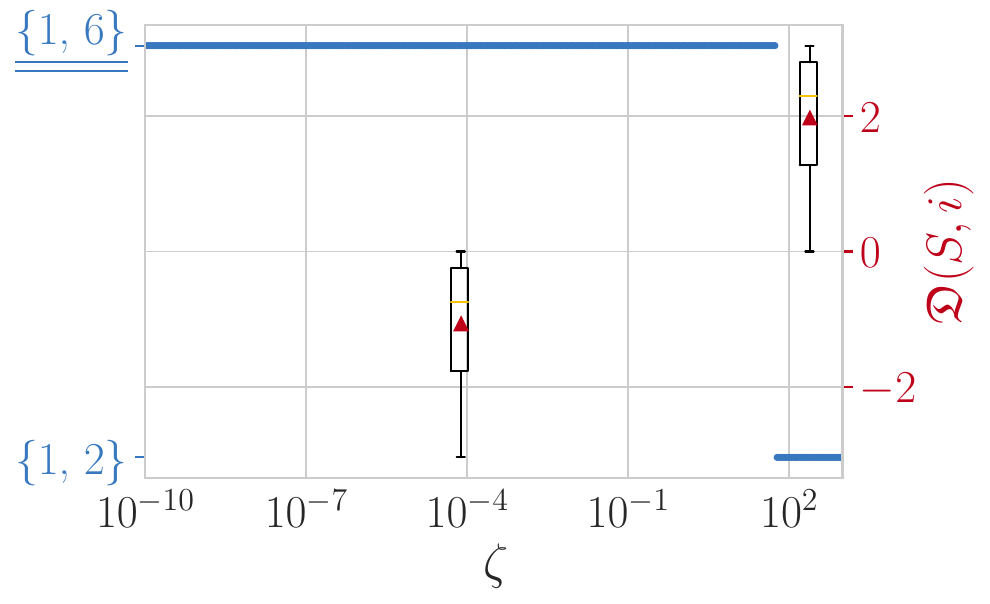}
\label{fig:d5-regu-path}}
\hfill
\subfloat[][]
{\includegraphics[width=0.32\textwidth]{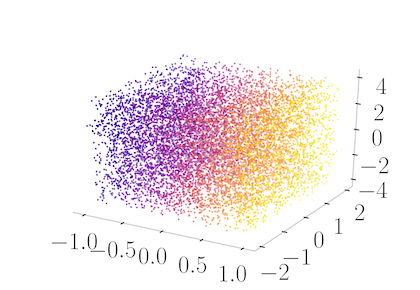}
\label{fig:d6-orig-data}}\hfill
\subfloat[][]
{\includegraphics[width=0.32\textwidth]{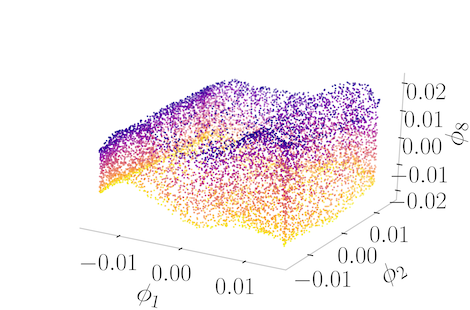}
\label{fig:d6-emb-rank-1}}\hfill
\subfloat[][]
{\includegraphics[width=0.32\textwidth]{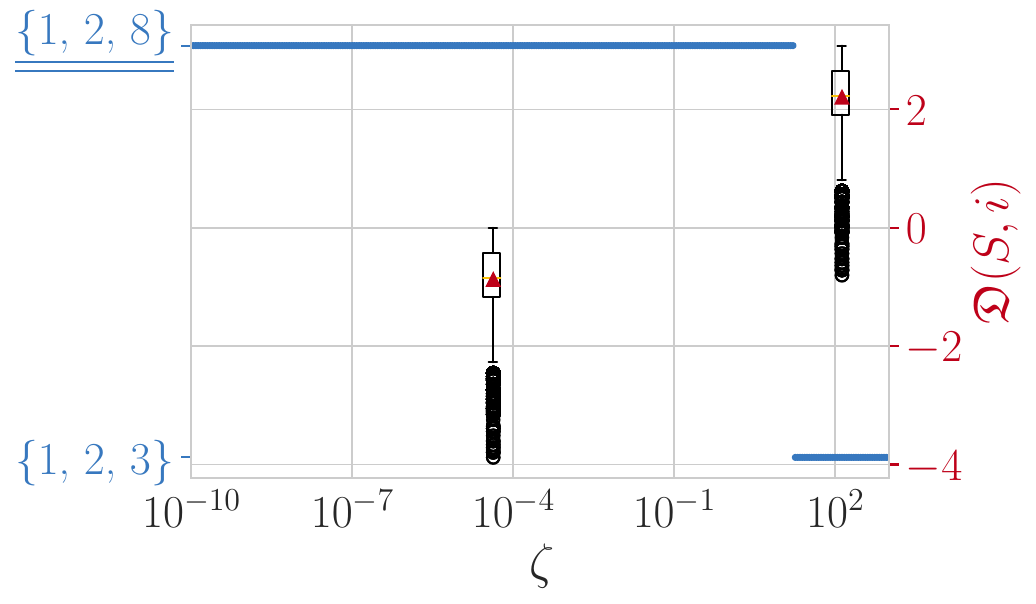}
\label{fig:d6-regu-path}}
\hfill

\caption{Synthetic manifolds with minimum embedding dimension $s$ equals intrinsic dimension $d$. Rows from top to bottom represent {\em two dimensional strip with cavity }(aspect ratio $W/H = 2\pi$), {\em swiss roll}, {\em gaussian manifold} and {\em three dimensional cube} dataset, respectively. Columns from left to right are the original data $\vect{X}$, embedding ${\phi}_{S_*}$ with optimal coordinate sets $S_*$ chosen by \coordsearch~and the regularization path, respectively.}
\label{fig:synth-data-s-d-equal-addi}
\end{figure}

\subsection{Additional experiments on synthetic manifolds with $s > d$}
\begin{figure*}[htb]
\subfloat[][]
{\includegraphics[width=0.32\textwidth]{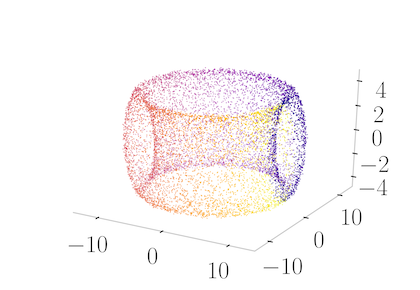}
\label{fig:d8-orig-data}}\hfill
\subfloat[][]
{\includegraphics[width=0.32\textwidth]{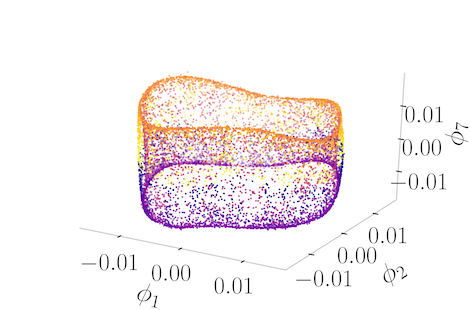}
\label{fig:d8-emb-rank-1}}\hfill
\subfloat[][]
{\includegraphics[width=0.32\textwidth]{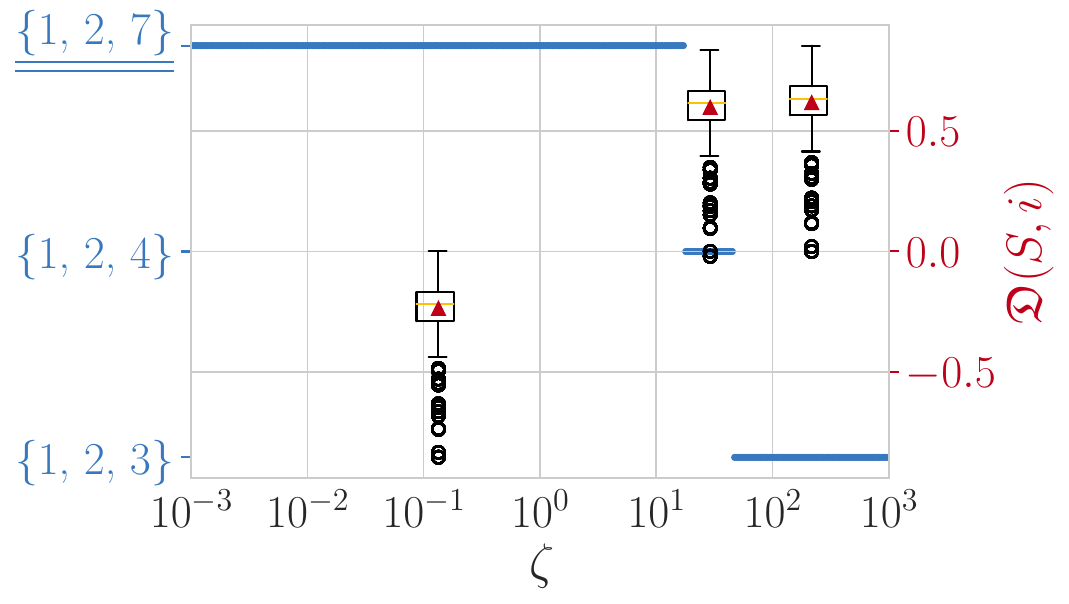}
\label{fig:d8-regu-path}}
\hfill
\subfloat[][]
{\includegraphics[width=0.32\textwidth]{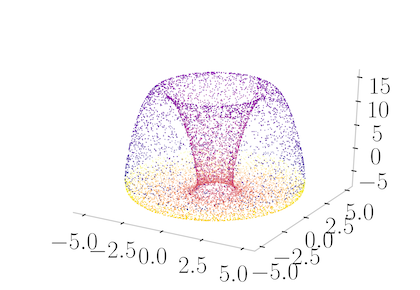}
\label{fig:d9-orig-data}}\hfill
\subfloat[][]
{\includegraphics[width=0.32\textwidth]{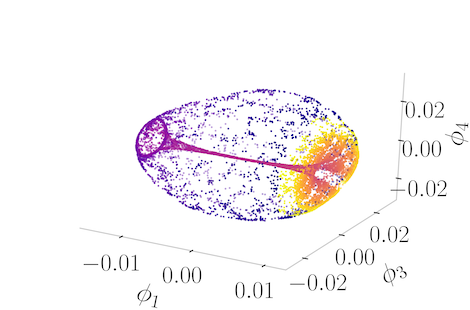}
\label{fig:d9-emb-rank-1}}\hfill
\subfloat[][]
{\includegraphics[width=0.32\textwidth]{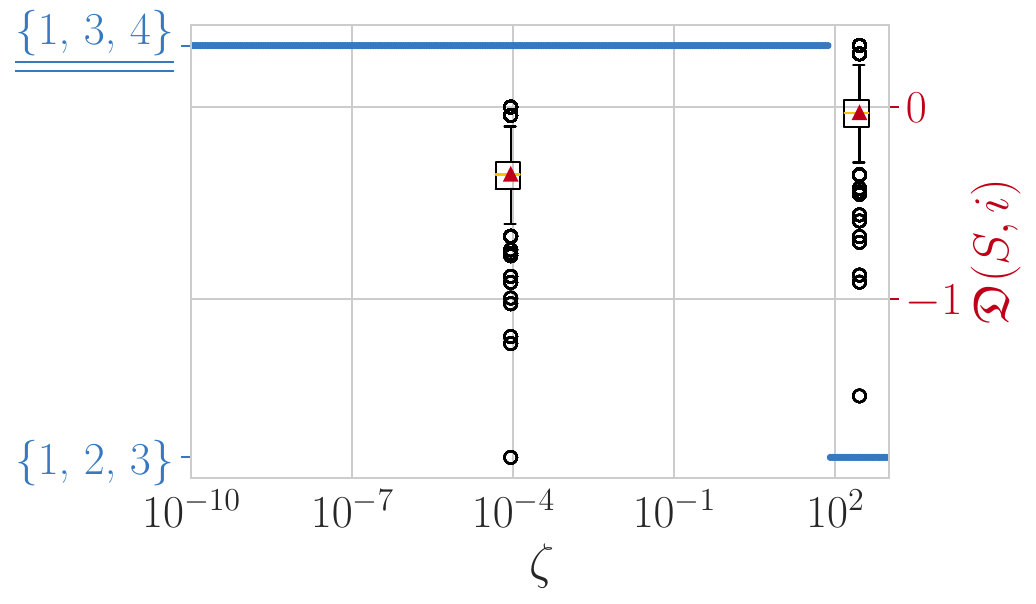}
\label{fig:d9-regu-path}}
\hfill
\subfloat[][]
{\includegraphics[width=0.32\textwidth]{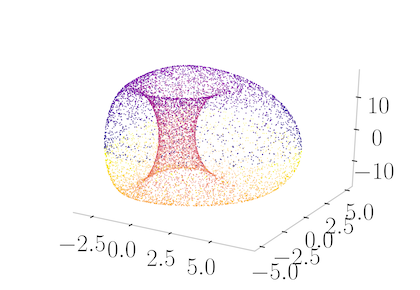}
\label{fig:d10-orig-data}}\hfill
\subfloat[][]
{\includegraphics[width=0.32\textwidth]{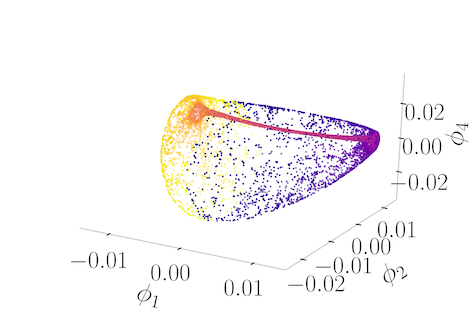}
\label{fig:d10-emb-rank-1}}\hfill
\subfloat[][]
{\includegraphics[width=0.32\textwidth]{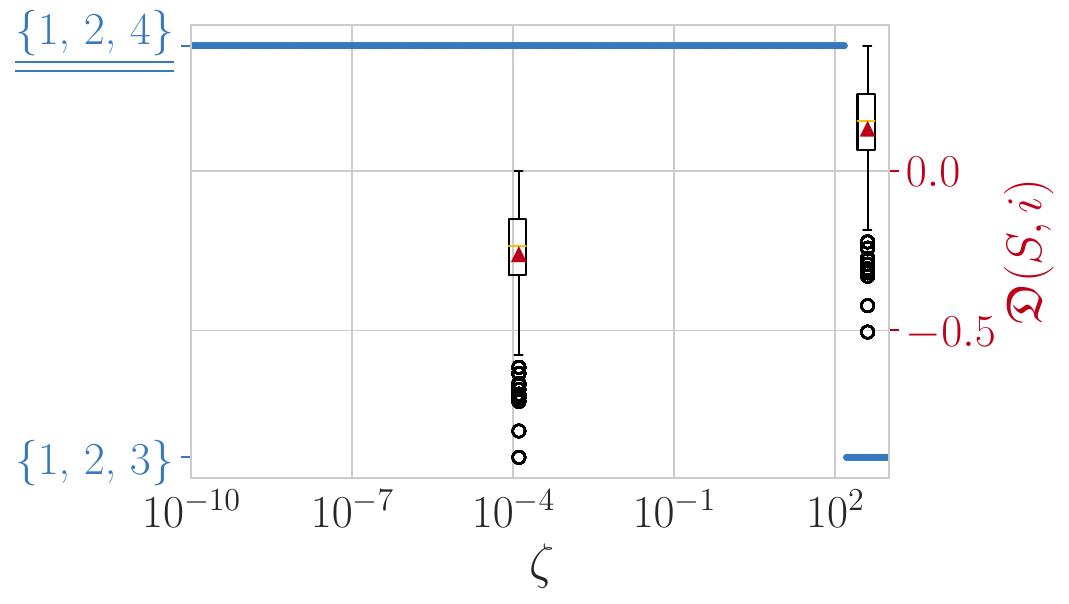}
\label{fig:d10-regu-path}}
\hfill
\subfloat[][]
{\includegraphics[width=0.32\textwidth]{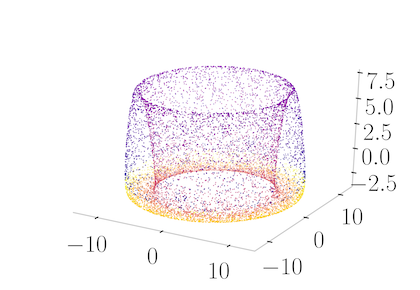}
\label{fig:d11-orig-data}}\hfill
\subfloat[][]
{\includegraphics[width=0.32\textwidth]{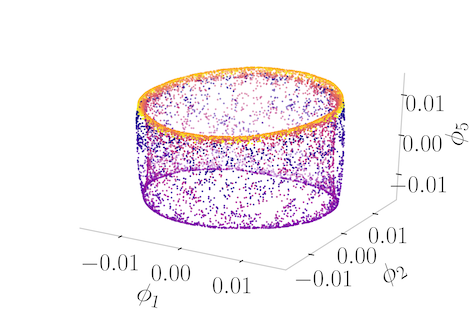}
\label{fig:d11-emb-rank-1}}\hfill
\subfloat[][]
{\includegraphics[width=0.32\textwidth]{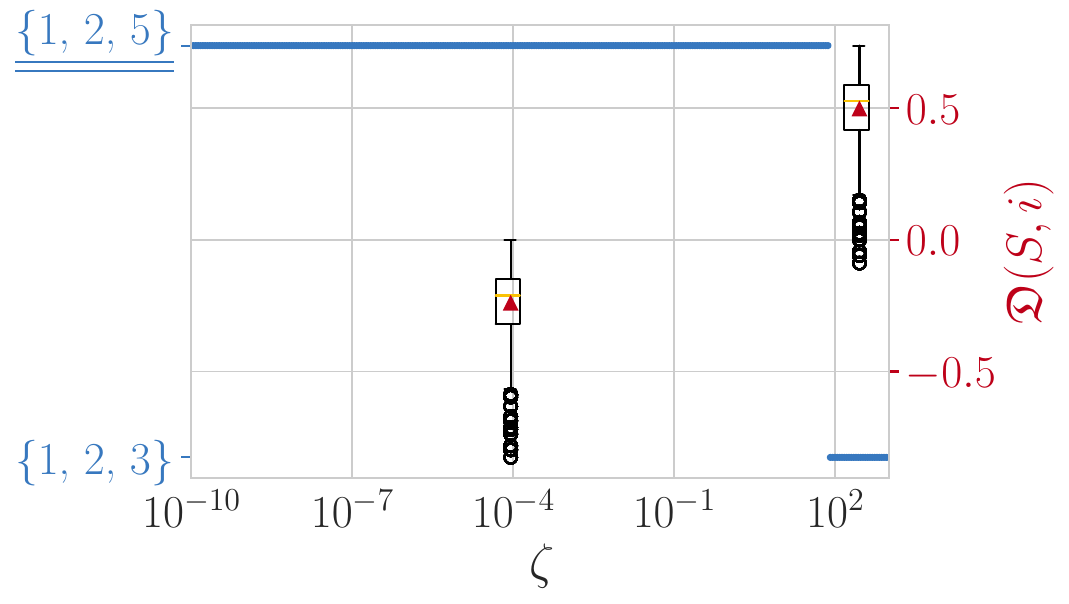}
\label{fig:d11-regu-path}}
\hfill
\subfloat[][]
{\includegraphics[width=0.32\textwidth]{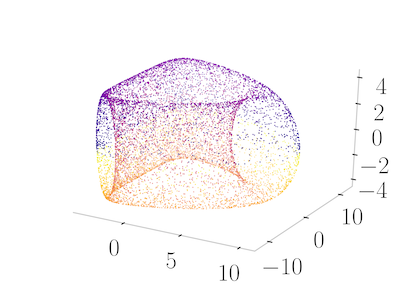}
\label{fig:d12-orig-data}}\hfill
\subfloat[][]
{\includegraphics[width=0.32\textwidth]{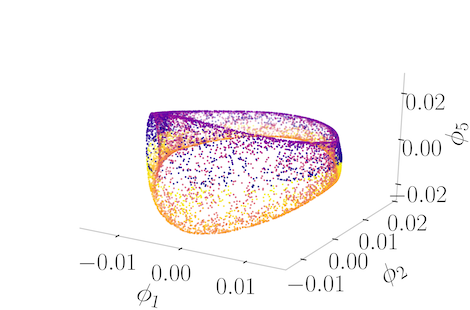}
\label{fig:d12-emb-rank-1}}\hfill
\subfloat[][]
{\includegraphics[width=0.32\textwidth]{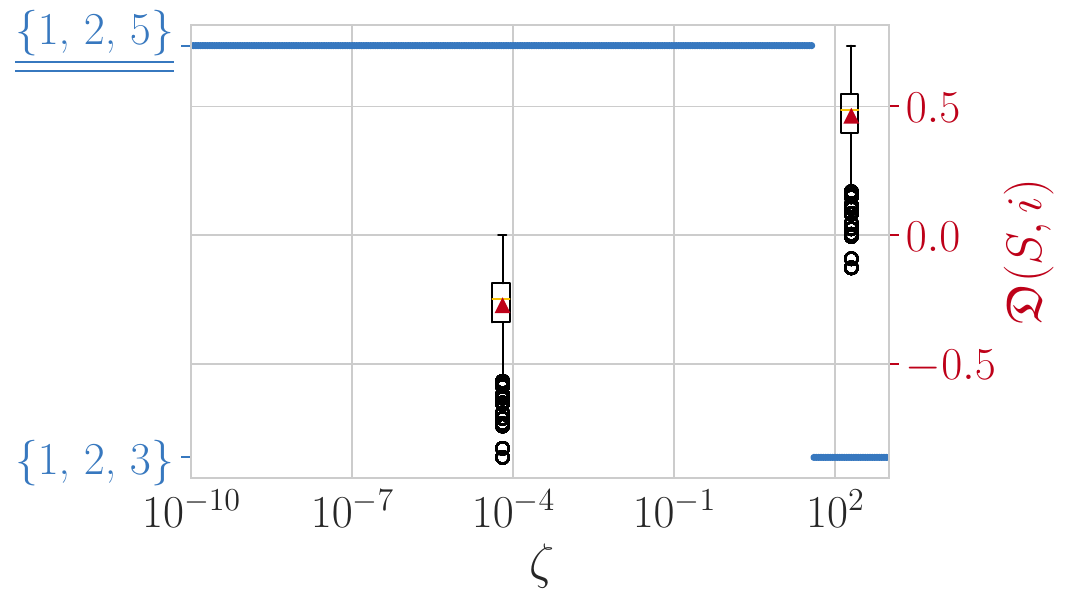}
\label{fig:d12-regu-path}}
\hfill

\caption{Synthetic manifolds with minimum embedding dimension $s$ greater than intrinsic dimension $d$. Rows from top to bottom represent {\em wide torus}, {\em z-asymmetrized high torus}, {\em x-asymmetrized high torus}, {\em z-asymmetrized wide torus} and {\em x-asymmetrized wide torus}, respectively. Columns from left to right are the original data $\vect{X}$, embedding ${\phi}_{S_*}$ with optimal coordinate sets $S_*$ chosen by \coordsearch~and the regularization path, respectively.}
\label{fig:synth-data-s-geq-d-addi}
\end{figure*}

\begin{figure*}[htb]
\centering
\subfloat[][]
{\includegraphics[width=0.97\textwidth]{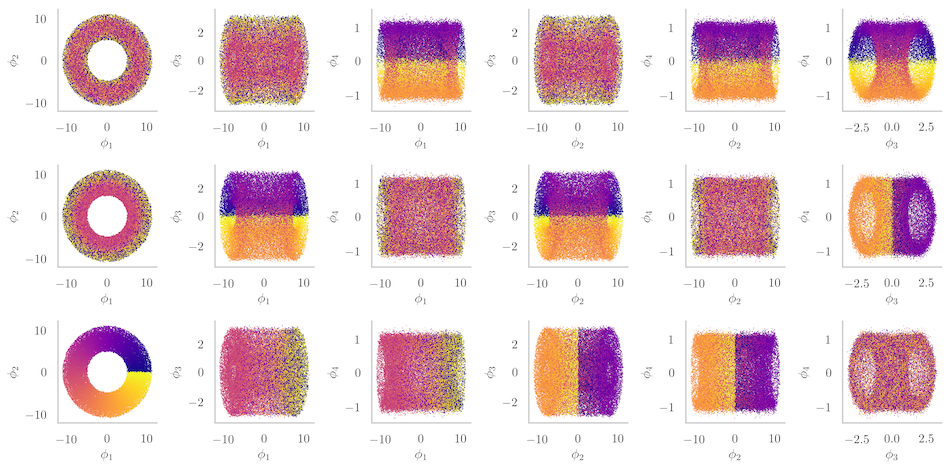}
\label{fig:d14-orig-data}}

\subfloat[][]
{\includegraphics[width=0.97\textwidth]{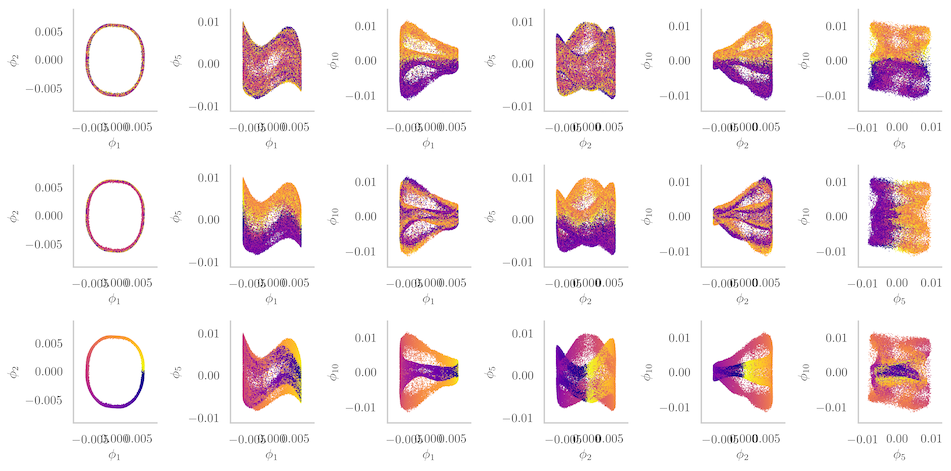}
\label{fig:d14-emb-rank-1}}
\caption{Experiment on {\em three-torus} dataset. (a) Original data $\vect{X}$ of three torus. (b) Embedding ${\phi}_{S_*}$ with optimal coordinate sets $S_*$ chosen by \coordsearch. Rows for both (a) and (b) from top to bottom are embedding colored by the parameterization $(\alpha_1, \alpha_2, \alpha_3)$ in \eqref{eq:three-torus-parameterization}, respectively.}
\label{fig:synth-data-3-torus}
\end{figure*}

\subsubsection{Tori and asymmetrized tori}

A torus can be parametrized by
\begin{equation}
\begin{split}
	x &= (a + b\cos\alpha)\cos\beta \\
	y &= (a + b\cos\alpha)\sin\beta\\
	z &= h\sin(\beta)
\end{split}
\label{eq:torus-parameterization}
\end{equation}

\begin{enumerate}
	\item $\mathcal D_7^*$: sampling $\vect{\alpha}, \vect{\beta}$ uniformly from $[0, 2\pi)$ and generating the torus with $(a, b, h) = (3, 2, 8)$ from \eqref{eq:torus-parameterization}. 
	\item $\mathcal D_8$: generating the torus with $(a, b, h) = (10, 2, 2)$. 
	\item $\mathcal D_9$: generating a high torus with $(a, b, h) = (3, 2, 8)$ and applying the 
		following transformation
		\begin{equation}
			z \gets (z - \min(z))^\gamma / \varsigma \\
			\label{eq:z-asym-trans}
		\end{equation}
		with $(\gamma, \varsigma) = (3, 1500)$
	\item $\mathcal D_{10}$: generating a high torus with $(a, b, h) = (3, 2, 8)$ and applying the 
		following transformation
		\begin{equation}
			x \gets (x - \min(x))^\kappa / \eta \\
			\label{eq:x-asym-trans}
		\end{equation}
		with $(\kappa, \eta) = (2, 10)$
	\item $\mathcal D_{11}$: generating a wide torus with $(a, b, h) = (10, 2, 2)$ and applying transformation \eqref{eq:z-asym-trans} with  $(\gamma, \varsigma) = (3, 50)$. 
	\item $\mathcal D_{12}$: generating a wide torus with $(a, b, h) = (10, 2, 2)$ and applying transformation \eqref{eq:x-asym-trans} with$(\kappa, \eta) = (3, 1000)$. 
\end{enumerate}

The experimental results are in Figure \ref{fig:synth-data-s-geq-d-addi}.

\subsubsection{Three-torus}
\begin{wrapfigure}[5]{r}{0.5\textwidth}
\centering
\includegraphics[width=0.9\linewidth]{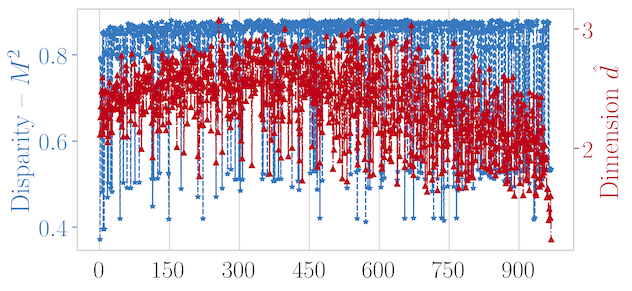}
\caption{$M^2$ and $\hat{d}$ vs. ranking of $\mathcal D_{13}$}
\label{fig:d14-m2-dim}
\end{wrapfigure}

The parameterization of the three torus is
\begin{equation}
\begin{split}
	x_1 & = a_1 \sin\alpha_1 \\
	x_2 & = (a_2 + a_1 \cos\alpha_1)\sin\alpha_2 \\
	x_3 & = (a_3 + (a_2 + a_1 \cos\alpha_1)\cos\alpha_2)\sin\alpha_3 \\
	x_4 & = (a_3 + (a_2 + a_1 \cos\alpha_1)\cos\alpha_2)\cos\alpha_3
\end{split}
\label{eq:three-torus-parameterization}
\end{equation}

To generate $\mathcal D_{13}$, we sample $\vect{\alpha}_k$ uniformly from $[0, 2\pi)$ for $k\in[3]$ and apply the transformation \eqref{eq:three-torus-parameterization} with $(a_1, a_2, a_3) = (8, 2, 1)$. The sample size for this dataset is $n = 50,000$. The experimental result of three-torus can be found in Figure \ref{fig:synth-data-3-torus}. 
    
\begin{figure}[htb]
\subfloat[][$\mathcal D_1$]
{\includegraphics[width=0.32\textwidth]{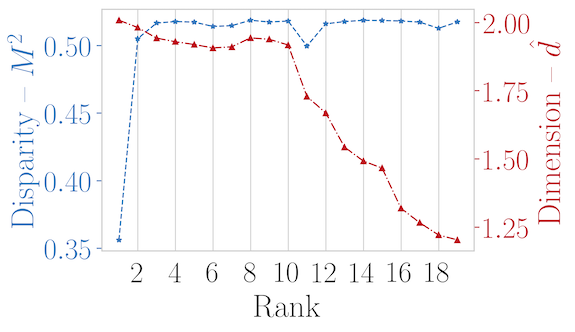}
\label{fig:d1-m2-dim}}\hfill
\subfloat[][$\mathcal D_2$]
{\includegraphics[width=0.32\textwidth]{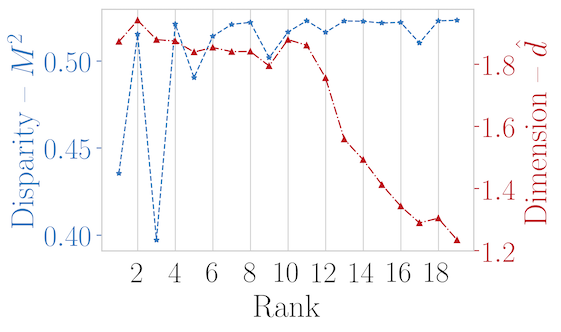}
\label{fig:d2-m2-dim}}\hfill
\subfloat[][$\mathcal D_3$]
{\includegraphics[width=0.32\textwidth]{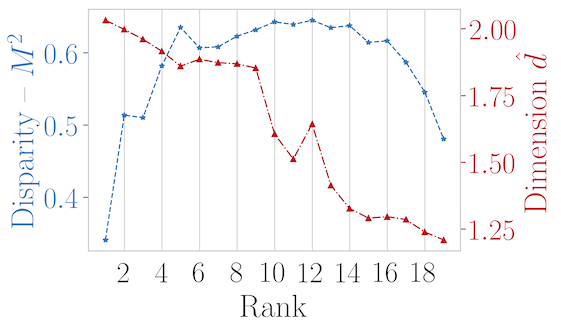}
\label{fig:d3-m2-dim}}
\hfill
\subfloat[][$\mathcal D_4$]
{\includegraphics[width=0.32\textwidth]{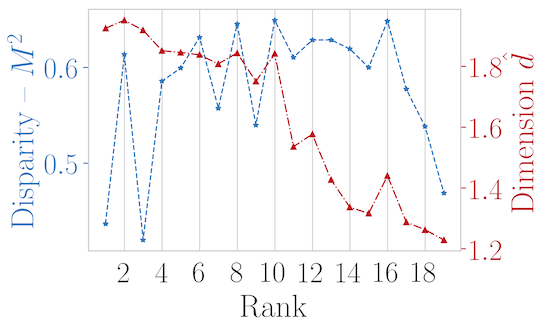}
\label{fig:d4-m2-dim}}\hfill
\subfloat[][$\mathcal D_5$]
{\includegraphics[width=0.32\textwidth]{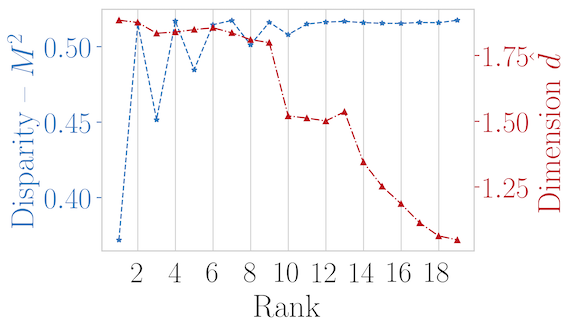}
\label{fig:d5-m2-dim}}\hfill
\subfloat[][$\mathcal D_6$]
{\includegraphics[width=0.32\textwidth]{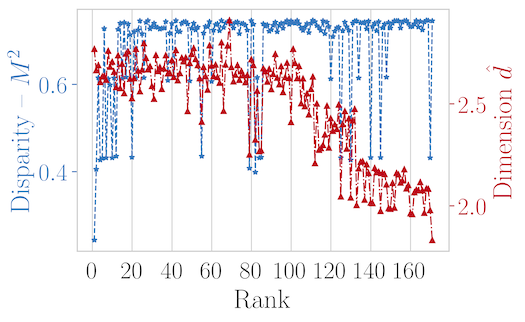}
\label{fig:d6-m2-dim}}
\hfill
\subfloat[][$\mathcal D_7$]
{\includegraphics[width=0.32\textwidth]{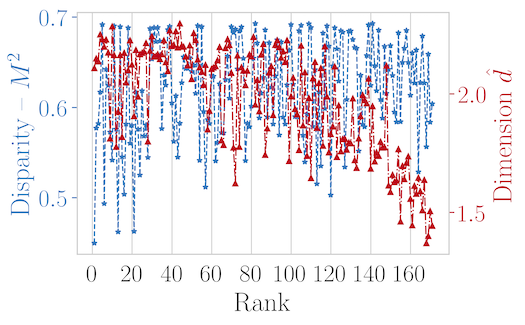}
\label{fig:d7-m2-dim}}\hfill
\subfloat[][$\mathcal D_8$]
{\includegraphics[width=0.32\textwidth]{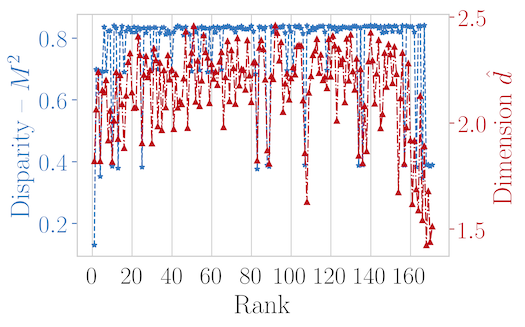}
\label{fig:d8-m2-dim}}\hfill
\subfloat[][$\mathcal D_9$]
{\includegraphics[width=0.32\textwidth]{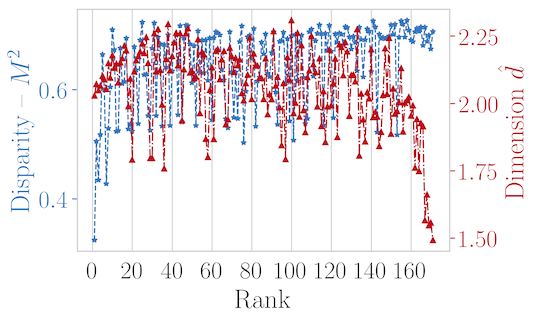}
\label{fig:d9-m2-dim}}
\hfill
\subfloat[][$\mathcal D_{10}$]
{\includegraphics[width=0.32\textwidth]{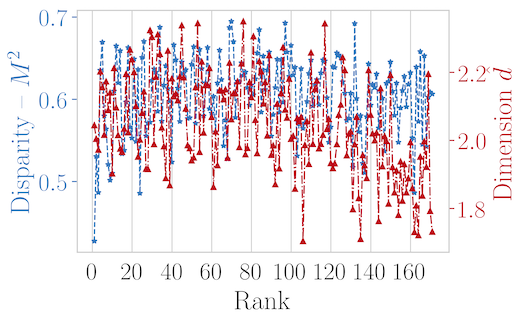}
\label{fig:d10-m2-dim}}\hfill
\subfloat[][$\mathcal D_{11}$]
{\includegraphics[width=0.32\textwidth]{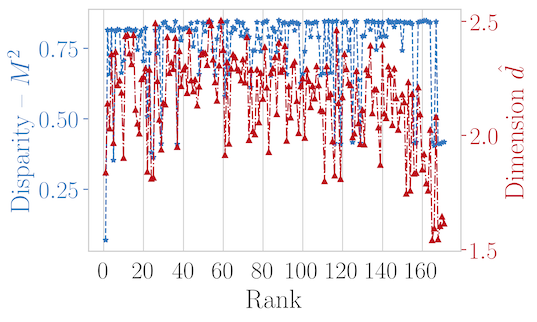}
\label{fig:d11-m2-dim}}\hfill
\subfloat[][$\mathcal D_{12}$]
{\includegraphics[width=0.32\textwidth]{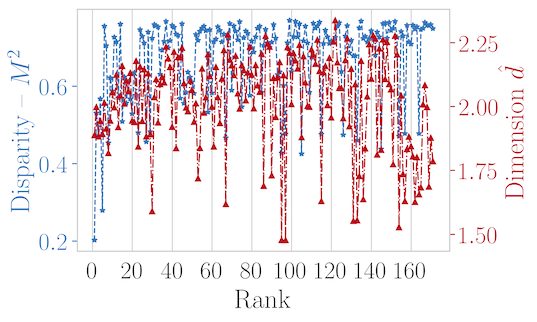}
\label{fig:d12-m2-dim}}
\hfill

\caption{Verification of the correctness of the chosen sets in synthetic manifolds.}
\label{fig:fig-all-m2-dim-verify}
\end{figure}

\subsection{Verification of the chosen subsets on synthetic manifolds}
Unlike 2D strip, the close form solution of the optimal set is oftentimes unknown in general. 
In this section, we verify the correctness of the chosen subset by
reporting the full procrustes distance (disparity score) $M^2$ \cite{DrydenI.L.IanL.2016Ssa:},
which is defined to be the normalized sum of square of the point-wise
difference between the procrustes transformed ground truth data $\vect{X}_\mathrm{true} \in\rrr^{n\times k}$
and the test data $\vect{X}_\mathrm{test}\in\rrr^{n\times k}$. Namely,

\begin{equation}
\begin{split}
	M^2(\vect{X}_\mathrm{true}, \vect{X}_\mathrm{test}) & = \min_{\beta,\vect{\gamma},\vect{\Gamma}} \|\vect{X}_\mathrm{true} - \beta \vect{X}_\mathrm{test}\vect{\Gamma} - \vect{1}_n\vect{\gamma}^\top\|_F^2 \\
	\text{ s.t. } & \beta>0, \vect{\gamma}\in\mathbb R^k, \Gamma\in SO(k)
\end{split}
\end{equation}

\begin{wrapfigure}[11]{r}{0.42\textwidth}
\vspace{-12pt}
\centering
\includegraphics[width=0.37\textwidth]{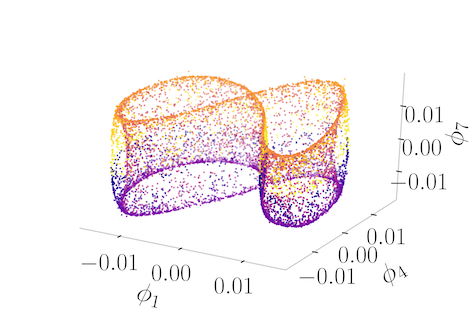}
\caption{Embedding that has crossing.}
\label{fig:d8-emb-rank-2}

\end{wrapfigure}
Here $\beta$ is a scale parameter, $\vect{\gamma}$ is the centering parameter and $\vect{\Gamma}$
is a $k\times k$ rotation matrix. We further require $\|\vect{X}_\mathrm{true}\|_F = 1$ so that 
the disparity score will be between 0 and 1. 
Intuitively, one can expect the optimal choice of eigencoordinates $S_*$ will yield a small
disparity score $M^2(\vect{X}_\mathrm{true}, \phi_{S_*})$, with score increases as the 
coordinate set $S$ contains duplicate 
parameterizations or $\phi_S$ contains {\em knots, crossings}, etc. (e.g., Figure \ref{fig:d8-emb-rank-2}). 
Note that the score can only be calculated when the ground truth data $\vect{X}_\mathrm{true}$ 
is available. 
For dataset without obtainable ground truth, one cannot proposed to 
report the disparity score of $\phi_{S}$ and the original data $\vect{X}$ 
as the proxy of $\vect{X}_\mathrm{true}$,
for $\vect{X}$ might not be a affine transformation of $\vect{X}_\mathrm{true}$, e.g., Swiss roll.
Besides, small $M^2$ given $\phi_S$ does not imply $S$ is optimal, which 
will be clear in the discussion of Figure \ref{fig:d4-emb-rank-3}. 
Besides disparity scores, we will also report the estimated dimension $\hat{d}$. One can expect
the estimated dimension for the optimal set $\dim(\phi_{S_*})$ will be close to the intrinsic 
dimension $d$, while the estimated dimension for sets containing duplicate parameterizations 
will be smaller than the 
intrinsic dimension. One cannot propose to use it as a criterion to choose the optimal set,
for the suboptimal sets can also have estimated dimensions closed to the intrinsic dimension, e.g.,
Figure \ref{fig:d7-llr-rank-1}. 
Throughout the experiment, the dimension estimation method by \cite{levina2005maximum} is 
used for its ability to estimate dimension among all candidate subsets fairly fast.
Blue and red curves in Figure \ref{fig:fig-all-m2-dim-verify} and \ref{fig:d14-m2-dim} 
show the disparity scores and 
estimated 
dimensions versus ranking of coordinate subsets for different synthetic manifolds, respectively. 
As expected, 
we have an increasing in $M^2$ and decreasing in $\hat{d}$ with respect to ranking.
We first highlight that the set that produces the lowest disparity score is not necessarily
optimal, although $S_*$ does yield a small disparity. This can be shown in the example
of $\mathcal D_4$ {\em swiss roll with hole} dataset.
Figure \ref{fig:d4-emb-rank-3} is the embedding ${\phi}_{S_3}$ of $\mathcal D_4$, 
with $S_3$ is ranked third subset in terms of $\loss(S;\zeta)$, that minimizes the
disparity score $M^2$ in $\mathcal D_4$ as shown in Figure \ref{fig:d4-m2-dim}.
This is because the embedding of the subset 
$S_3 = \{1, 11\}$ has larger area on the left, 
compared to Figure Figure \ref{fig:d4-emb-rank-1}.
This balances out the high disparity caused by the 
{\em flipped} region between two {\em knots} in the embedding ${\phi}_{S_3}$ 
when matched with $\vect{X}_\mathrm{true}$. Since all the ranked first subset has low disparity
compared to other subsets, we have higher confidence saying that the ranked 1st subset is indeed
the optimal choice for the synthetic manifolds.

%% file: ms.bbl
\begin{thebibliography}{AAMA{\etalchar{+}}09}

\bibitem[AAMA{\etalchar{+}}09]{abazajian2009seventh}
Kevork~N Abazajian, Jennifer~K Adelman-McCarthy, Marcel~A Ag{\"u}eros, Sahar~S
  Allam, Carlos~Allende Prieto, Deokkeun An, Kurt~SJ Anderson, Scott~F
  Anderson, James Annis, Neta~A Bahcall, et~al.
\newblock The seventh data release of the sloan digital sky survey.
\newblock {\em The Astrophysical Journal Supplement Series}, 182(2):543, 2009.

\bibitem[Bat14]{bates:16}
Jonathan Bates.
\newblock The embedding dimension of laplacian eigenfunction maps.
\newblock {\em Applied and Computational Harmonic Analysis}, 37(3):516--530,
  2014.

\bibitem[BN07]{belkin:07}
Mikhail Belkin and Partha Niyogi.
\newblock Convergence of laplacian eigenmaps.
\newblock In B.~Sch\"{o}lkopf, J.~C. Platt, and T.~Hoffman, editors, {\em
  Advances in Neural Information Processing Systems 19}, pages 129--136. MIT
  Press, 2007.

\bibitem[CL06]{coifman:06}
R.~R. Coifman and S.~Lafon.
\newblock Diffusion maps.
\newblock {\em Applied and Computational Harmonic Analysis}, 30(1):5--30, 2006.

\bibitem[CTS{\etalchar{+}}17]{chmiela2017machine}
Stefan Chmiela, Alexandre Tkatchenko, Huziel~E Sauceda, Igor Poltavsky,
  Kristof~T Sch{\"u}tt, and Klaus-Robert M{\"u}ller.
\newblock Machine learning of accurate energy-conserving molecular force
  fields.
\newblock {\em Science advances}, 3(5):e1603015, 2017.

\bibitem[Dry16]{DrydenI.L.IanL.2016Ssa:}
I.~L. (Ian~L.) Dryden.
\newblock {\em Statistical shape analysis : with applications in R}.
\newblock Wiley series in probability and statistics. Wiley, Chichester, West
  Sussex, England, 2nd ed. edition, 2016.

\bibitem[DS13]{dasgupta2013randomized}
Sanjoy Dasgupta and Kaushik Sinha.
\newblock Randomized partition trees for exact nearest neighbor search.
\newblock In {\em Conference on Learning Theory}, pages 317--337, 2013.

\bibitem[DTCK18]{dsilva2018parsimonious}
Carmeline~J Dsilva, Ronen Talmon, Ronald~R Coifman, and Ioannis~G Kevrekidis.
\newblock Parsimonious representation of nonlinear dynamical systems through
  manifold learning: A chemotaxis case study.
\newblock {\em Applied and Computational Harmonic Analysis}, 44(3):759--773,
  2018.

\bibitem[FTP16]{flemingTPfae:16}
Kelly~L. Fleming, Pratyush Tiwary, and Jim Pfaendtner.
\newblock New approach for investigating reaction dynamics and rates with ab
  initio calculations.
\newblock {\em Jornal of Physical Chemistry A}, 120(2):299--305, 2016.

\bibitem[GZKR08]{goldberg08}
Yair Goldberg, Alon Zakai, Dan Kushnir, and Ya’acov Ritov.
\newblock Manifold learning: The price of normalization.
\newblock {\em Journal of Machine Learning Research}, 9(Aug):1909--1939, 2008.

\bibitem[Har98]{harville1998matrix}
David~A Harville.
\newblock Matrix algebra from a statistician's perspective, 1998.

\bibitem[HAvL05]{HeinAL:05}
Matthias Hein, Jean{-}Yves Audibert, and Ulrike von Luxburg.
\newblock From graphs to manifolds - weak and strong pointwise consistency of
  graph laplacians.
\newblock In {\em Learning Theory, 18th Annual Conference on Learning Theory,
  {COLT} 2005, Bertinoro, Italy, June 27-30, 2005, Proceedings}, pages
  470--485, 2005.

\bibitem[HAvL07]{HeinAL:07}
Matthias Hein, Jean{-}Yves Audibert, and Ulrike von Luxburg.
\newblock Graph laplacians and their convergence on random neighborhood graphs.
\newblock {\em Journal of Machine Learning Research}, 8:1325--1368, 2007.

\bibitem[HHJ90]{horn1990matrix}
Roger~A Horn, Roger~A Horn, and Charles~R Johnson.
\newblock {\em Matrix analysis}.
\newblock Cambridge university press, 1990.

\bibitem[IB12]{Iyer:2012:AAM:3020652.3020697}
Rishabh Iyer and Jeff Bilmes.
\newblock Algorithms for approximate minimization of the difference between
  submodular functions, with applications.
\newblock In {\em Proceedings of the Twenty-Eighth Conference on Uncertainty in
  Artificial Intelligence}, UAI'12, pages 407--417, Arlington, Virginia, United
  States, 2012. AUAI Press.

\bibitem[LB05]{levina2005maximum}
Elizaveta Levina and Peter~J Bickel.
\newblock Maximum likelihood estimation of intrinsic dimension.
\newblock In {\em Advances in neural information processing systems}, pages
  777--784, 2005.

\bibitem[Lee03]{LeeJohnM.2003Itsm}
John~M. Lee.
\newblock Introduction to smooth manifolds, 2003.

\bibitem[MHM18]{mcinnes2018umap}
Leland McInnes, John Healy, and James Melville.
\newblock Umap: Uniform manifold approximation and projection for dimension
  reduction.
\newblock {\em arXiv preprint arXiv:1802.03426}, 2018.

\bibitem[MMVZ16]{JMLR:v17:16-109}
James McQueen, Marina Meil\u{a}, Jacob VanderPlas, and Zhongyue Zhang.
\newblock Megaman: Scalable manifold learning in python.
\newblock {\em Journal of Machine Learning Research}, 17(148):1--5, 2016.

\bibitem[NLCK06]{nadler:06}
Boaz Nadler, Stephane Lafon, Ronald Coifman, and Ioannis Kevrekidis.
\newblock Diffusion maps, spectral clustering and eigenfunctions of
  {Fokker-Planck} operators.
\newblock In Y.~Weiss, B.~Sch\"{o}lkopf, and J.~Platt, editors, {\em Advances
  in Neural Information Processing Systems 18}, pages 955--962, Cambridge, MA,
  2006. MIT Press.

\bibitem[NWF78]{nemhauser1978analysis}
George~L Nemhauser, Laurence~A Wolsey, and Marshall~L Fisher.
\newblock An analysis of approximations for maximizing submodular set
  functions—i.
\newblock {\em Mathematical programming}, 14(1):265--294, 1978.

\bibitem[PM13]{2013arXiv1305.7255P}
D.~{Perraul-Joncas} and M.~{Meila}.
\newblock {Non-linear dimensionality reduction: Riemannian metric estimation
  and the problem of geometric discovery}.
\newblock {\em ArXiv e-prints}, May 2013.

\bibitem[Por16]{Portegies:16}
Jacobus~W Portegies.
\newblock Embeddings of riemannian manifolds with heat kernels and
  eigenfunctions.
\newblock {\em Communications on Pure and Applied Mathematics}, 69(3):478--518,
  2016.

\bibitem[Str07]{strauss2007partial}
Walter~A Strauss.
\newblock {\em Partial differential equations: An introduction}.
\newblock Wiley, 2007.

\bibitem[THJ10]{TingHJ:10}
Daniel Ting, Ling Huang, and Michael~I. Jordan.
\newblock An analysis of the convergence of graph laplacians.
\newblock In {\em Proceedings of the 27th International Conference on Machine
  Learning (ICML-10)}, pages 1079--1086, 2010.

\end{thebibliography}
